\documentclass[a4paper, fleqn]{cas-dc}

\usepackage[utf8]{inputenc}
\usepackage[T1]{fontenc}
\usepackage[numbers, sort]{natbib}
\usepackage{amsmath}
\usepackage{amssymb}
\usepackage{amsfonts}
\usepackage{amsthm}
\usepackage{mathtools}
\usepackage{graphicx}
\usepackage{realboxes}
\usepackage{enumitem}
\usepackage{textcomp}
\usepackage{color}
\usepackage{soul}
\usepackage{bbm}
\usepackage{bm}
\usepackage{multirow}
\usepackage{booktabs}
\usepackage{subcaption}
\usepackage[dvipsnames, svgnames, x11names, table]{xcolor}
\usepackage{hyperref}
\usepackage{algorithm}
\usepackage[noend]{algpseudocode}
\usepackage{array}

\theoremstyle{definition}
\newtheorem{definition}{Definition}
\newtheorem*{definition*}{Definition}

\newtheorem{problem}{Problem}
\newtheorem{remark}{Remark}
\newtheorem{proposition}{Proposition}

\newtheorem{example}{Example}

\DeclareMathOperator*{\argmin}{arg\,min}
\DeclareMathOperator*{\interior}{int}
\renewcommand{\emptyset}{\text{\O}}             
\renewcommand{\epsilon}{\varepsilon}

\makeatletter
\let\save@mathaccent\mathaccent
\newcommand*\if@single[3]{%
     \setbox0\hbox{${\mathaccent"0362{#1}}^H$}%
     \setbox2\hbox{${\mathaccent"0362{\kern0pt#1}}^H$}%
     \ifdim\ht0=\ht2 #3\else #2\fi
}
\newcommand*\rel@kern[1]{\kern#1\dimexpr\macc@kerna}
\newcommand*\widebar[1]{\@ifnextchar^{{\wide@bar{#1}{0}}}{\wide@bar{#1}{1}}}
\newcommand*\wide@bar[2]{\if@single{#1}{\wide@bar@{#1}{#2}{1}}{\wide@bar@{#1}{#2}{2}}}
\newcommand*\wide@bar@[3]{%
     \begingroup
     \def\mathaccent##1##2{%
          \let\mathaccent\save@mathaccent
          \if#32 \let\macc@nucleus\first@char \fi
          \setbox\z@\hbox{$\macc@style{\macc@nucleus}_{}$}%
          \setbox\tw@\hbox{$\macc@style{\macc@nucleus}{}_{}$}%
          \dimen@\wd\tw@
          \advance\dimen@-\wd\z@
          \divide\dimen@ 3
          \@tempdima\wd\tw@
          \advance\@tempdima-\scriptspace
          \divide\@tempdima 10
          \advance\dimen@-\@tempdima
          \ifdim\dimen@>\z@ \dimen@0pt\fi
          \rel@kern{0.6}\kern-\dimen@
          \if#31
               \overline{\rel@kern{-0.6}\kern\dimen@\macc@nucleus\rel@kern{0.4}\kern\dimen@}%
               \advance\dimen@0.4\dimexpr\macc@kerna
               \let\final@kern#2%
               \ifdim\dimen@<\z@ \let\final@kern1\fi
               \if\final@kern1 \kern-\dimen@\fi
          \else
               \overline{\rel@kern{-0.6}\kern\dimen@#1}%
          \fi
     }%
     \macc@depth\@ne
     \let\math@bgroup\@empty \let\math@egroup\macc@set@skewchar
     \mathsurround\z@ \frozen@everymath{\mathgroup\macc@group\relax}%
     \macc@set@skewchar\relax
     \let\mathaccentV\macc@nested@a
     \if#31
     \macc@nested@a\relax111{#1}%
     \else
          \def\gobble@till@marker##1\endmarker{}%
          \futurelet\first@char\gobble@till@marker#1\endmarker
          \ifcat\noexpand\first@char A\else
               \def\first@char{}%
          \fi
          \macc@nested@a\relax111{\first@char}%
     \fi
     \endgroup
}
\makeatother
	
\begin{document}
\let\WriteBookmarks\relax
\def\floatpagepagefraction{1}
\def\textpagefraction{.001}
\shorttitle{Entanglement-Free Trajectory Planning for Tethered Robots}
\shortauthors{G. Battocletti et~al.}

\title [mode = title]{Entanglement-Free Trajectory Planning for Tethered Mobile Robots with a Slack Tether}

\author[]{Gianpietro Battocletti}[orcid=0009-0004-6981-0017]
\cormark[1]
\ead{g.battocletti@tudelft.nl}
\credit{Conceptualization of this work, Methodology, Software implementation, Writing -- Original draft preparation}
\affiliation[]{
    organization={Delft Center for Systems and Control, Delft University of Technology},
    city={Delft},
    country={The Netherlands}
}

\author[]{Dimitris Boskos}[orcid=0000-0003-0287-9197]
\ead{d.boskos@tudelft.nl}
\credit{Supervision, Methodology, Review}

\author[]{Bart {De Schutter}}[orcid=0000-0001-9867-6196]
\ead{b.deschutter@tudelft.nl}
\credit{Supervision, Methodology, Review}

\cortext[cor1]{Corresponding author}

\begin{abstract}
    In motion planning algorithms for tethered mobile robots, the entanglement state of the tether is a critical aspect to consider during the planning phase. 
    This is particularly important in case of a slack tether, where the shape of the tether is not determined solely by the geometry of the environment and the location of the obstacles, but also by the dynamics of the tether, by the trajectory followed by the robot, and possibly by exogenous forces.
    In this scenario, preventing entanglement requires planning a robot trajectory that accounts for the entanglement definition and for the dynamics of the robot and of the tether.
    In this work, we propose a motion planning algorithm for tethered mobile robots with a slack tether that computes dynamically feasible entanglement-free trajectories to navigate through an environment with static obstacles. 
    By considering the entanglement state during all the stages of the planning pipeline, we are able to compute safer trajectories that avoid entanglement during the motion of the robot.
    We achieve this through a three-step pipeline, which includes (i) the construction of a topological model of the entanglement-free configuration space of the tethered robot, (ii) the generation of a set of candidate paths using this model, and (iii) the computation of a dynamically feasible entanglement-free trajectory by solving a homotopy-constrained trajectory generation problem.
    The resulting trajectory can then be executed to lead the robot to its target location, while maintaining the tether in an entanglement-free configuration.
    We demonstrate the benefits of this algorithm in simulations, where we show how the planning algorithm avoids violations of the entanglement constraints, resulting in safer and more reliable trajectories.
\end{abstract}

\begin{keywords}
	Tethered mobile robots \sep Entanglement avoidance \sep Path and motion planning \sep Trajectory optimization \sep Simplicial complex \sep Homotopy constraints
\end{keywords}

\maketitle

\section{Introduction}
\label{sec:introduction}
Tethered mobile robots are a class of mobile robots characterized by a cabled connection between the robot and another point in the environment, which can be either a ground station or another robot serving as a mobile hub \cite{marques2023tethered}. 
The tether can be used to provide power and to enable reliable communication between the robot and its connection point, making tethered mobile robots suitable for tasks with a long duration, such as search and rescue tasks \cite{pratt2008tethered}, inspection of civil infrastructures \cite{almhdawi2021cart, omandam20223d}, and litter collection from the seabed \cite{ilioudi2026seaclear}, as well as in communication-deprived environments, as in the case of underwater applications \cite{amer2025react,mccammon2017planning, buchholz2025framework}.
One critical risk in tethered mobile robots is that of tether entanglement, which can happen for example when the tether winds around obstacles, or forms knots with itself.
Entanglement of the tether can lead to suboptimal trajectories, requiring specific motions to be performed in order to disentangle the tether, or it can even result in the failure of the task and require the intervention of an external operator.
The definition of entanglement is in general application- and domain-dependent, with a comprehensive overview of the existing definitions being provided in \cite{battocletti2024entanglement}.

Entanglement prevention is particularly critical in the case of tethered mobile robots with a slack tether, i.e., where there is not a dedicated tethered management system that regulates the length of the tether over time.
When a tether management system is available, the tether is usually kept in a taut configuration, where `taut' indicates that the tether configuration is always a locally shortest path \cite{battocletti2024entanglement}.
In this scenario the tether configuration can be determined solely from the geometry of the workspace and from the path followed by the robot \cite{mcgarey2017tslam, teshnizi2021motion}. 
However, keeping the tether taut has some drawbacks: first, it requires the availability of a tether management system, along with a dedicated controller \cite{petit2022tape}; second, due to the internal tension that keeps it taut, the tether exerts a force on the robot that must be compensated \cite{tognon2017dynamics}; lastly, a taut tether bends around obstacles, which can damage it or lead it to get stuck \cite{rajan2016tether}.

Despite these disadvantages, in the majority of the literature a taut tether model is considered. 
The reason for this is that the assumption of a taut tether significantly simplifies the motion planning problem, as the evaluation of the entanglement state of the tether configuration becomes a purely geometric problem \cite{hert1999motion}.
In fact, when considering a taut tether, the entanglement definitions have usually a geometric nature, e.g., they require the tether to not form bends, or to avoid forming loops around obstacles and crossing with itself \cite{sinden1990tethered}.
Even in works that consider slack tether models, the entanglement state of the tether is evaluated on the tautened version of the tether configuration, i.e., on the shortest path in the homotopy class of the tether \cite{amer2025react, cao2023neptune, cao2025braid}, and entanglement state of the tether is usually characterized through geometric criteria.
Furthermore, existing approaches for motion planning for tethered robots with a slack tether are tailored to specific tether dynamics and entanglement definitions, lacking generality and flexibility, and cannot be used with different tether dynamic models and entanglement definitions.
For example, in \cite{shimada2024tangle} the dynamics of the tether is ignored entirely, while in \cite{shapovalov2020tangle, petit2022tape} only the geometric shape traced by a tether suspended between two points is considered.

When dealing with a general slack tether and a general entanglement definition, these simplifications are not possible, as the tether shape, and thus its entanglement state, depend on the trajectory followed by the robot, and on the tether dynamics. 
However, directly optimizing the trajectory based on the dynamics of the tether and of the robot is, in general, a computationally expensive problem, as due to the presence of obstacles, and to the high dimensionality of the tether configuration, the problem is highly nonlinear and nonconvex. 
Moreover, the presence of obstacles giving rise to multiple homotopy classes yields a solution space that is not connected, making the optimization problem even harder by trapping the solver in local minima, as solvers struggle to switch between solutions in different homotopy classes \cite{kim2012trajectory, he2022homotopy}.
This issue is amplified by the need for the planning algorithm to actively search through different homotopy classes, as some of them may not admit the existence of any entanglement-free trajectory. 
Still, multiple homotopy classes may yield feasible entanglement-free trajectories, among which the planner must choose when searching for a globally optimal trajectory.

In this work, we propose a novel approach to generate dynamically feasible entanglement-free trajectories for mobile tethered robots with a slack tether.
To do so, we take inspiration from \cite{zhang2026homotopy,degroot2024topology, rosmann2017integrated} and break the trajectory generation into two steps: we first compute a candidate path, which uniquely identifies a homotopy class, and then perform the trajectory optimization in that homotopy class, so that the solution space is now connected.
The process can be repeated multiple times using paths lying in different homotopy classes as initial solutions, allowing to exhaustively cover the solution space, and resulting in the computation of multiple locally optimal trajectories that can then be compared to select the best one \cite{sahin2024topo}.
This approach is particularly suited for motion planning for tethered mobile robots since, differently from the case of motion planning for untethered robots, homotopy-class constraints naturally emerge due to the presence of the tether. 
Even more, the search and selection of an appropriate homotopy class for the motion of the robot is usually a part of the path planning algorithm for tethered mobile robots. This means that there is only a limited additional overhead required for the parallel optimization of multiple trajectories in different homotopy classes.
Lastly, the use of entanglement definitions conveniently allows pruning the set of homotopy classes to consider, alleviating the computational burden of the trajectory optimization problem.

As a result of this approach, the proposed planning pipeline is composed by three main parts, depicted in Figure \ref{fig:overview} and described next:
\begin{enumerate}
    \item We start by generating a topological model of the configuration space of the tethered robot. During the construction of this model we verify both the tether length constraint and the existence of an entanglement-free tether configuration in each point of the model. In this way, it is possible to exclude homotopy classes where no feasible solution to the trajectory generation problem exists, reducing the search space;
    \item We then run a path planning algorithm on the topological model to obtain one or more entanglement-free paths. 
    While the paths are entanglement-free, they do not guarantee the existence of a corresponding dynamically feasible entanglement-free trajectory, and rather serve as initial solution candidates in different homotopy classes for the trajectory generation step.
    In fact, as mentioned, multiple solutions to the path planning problem may exist in distinct homotopy classes. The path planning algorithm can enumerate the existing solutions, selecting the most promising ones.
    Additionally, in case the robot finds itself in an entangled configuration, the proposed topological model can also be used to perform the disentanglement path planning operation, enabling the possibility for the robot to return to a safe state with an entanglement-free tether.
    \item Last, we compute a dynamically feasible solution by solving a homotopy-constrained trajectory generation problem. In this stage, we compute and evaluate multiple trajectories, corresponding to different homotopy classes \cite{sahin2024topo}. This approach has two relevant benefits, namely, (i) a larger part of the solution space is covered, since multiple homotopy classes are considered through the multiple candidate paths, increasing the likelihood to generate a trajectory in the homotopy class in which the globally optimal solution lies, and (ii) multiple alternatives are evaluated in case no dynamically feasible trajectory exists in some of the candidate homotopy classes. 
\end{enumerate}
The proposed pipeline is flexible and can be used with different types of robots and in different application domains, such as aerial, underwater, and space mobile robots. 
In this sense, the proposed approach allows including the appropriate dynamics of the robot and that of the tether depending on the target robotic platform and application domain.
Moreover, the components of the pipeline can be used as standalone modules, as well as in combination with other planning algorithms, providing a versatile planning framework.
\begin{figure*}
    \centering
    \includegraphics[width=\textwidth]{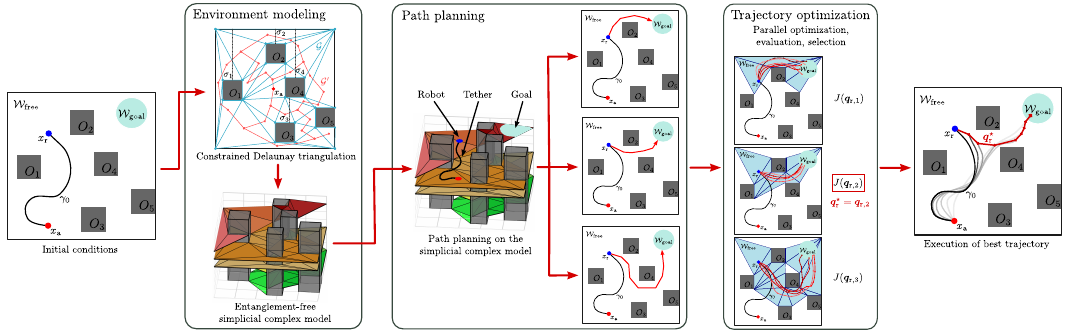}
    \caption{Overview of the motion planning pipeline developed in this work. First, a simplicial complex model of the entanglement-free configuration space of the tethered robot is computed. This model is then used to determine a set of candidate paths along which the robot can move while remaining in the entanglement-free configuration space. Last, the paths are converted into dynamically feasible trajectories by solving an optimization-based homotopy-constrained trajectory generation problem. The best trajectory is then executed by the robot.}
    \label{fig:overview}
\end{figure*}

The contributions of this work with respect to the state of the art are summarized as follows:
\begin{itemize}
    \item We introduce a novel algorithm to generate a topological model of the  entanglement-free configuration space of a tethered robot.
    To do so, we build on the simplicial complex model of \cite{battocletti2025efficient} by introducing entanglement avoidance constraints during the construction of the model.
    This model allows for efficient motion planning using a broad class of path planning methods, while removing the need to explicitly track the homotopy class of the path during the path planning operation.
    \item We propose a path planning pipeline that leverages the entan\-gle\-ment-free topological model to compute entan\-gle\-ment-free paths for tethered robots with a slack tether.
    In addition to finding the shortest entan\-gle\-ment-free path to reach the goal, the proposed path planning pipeline allows enumerating homotopically distinct entan\-gle\-ment-free paths.
    Furthermore, differently from existing methods, the proposed path planning framework is also suitable for disentanglement path planning, i.e., to find a path leading a robot from an initially entangled tether configuration to a non-entangled one.
    \item We propose a novel homotopy-constrained trajectory planning method for tethered robots that allows to compute dynamically feasible entan\-gle\-ment-free trajectories. This is achieved by solving an optimization-based trajectory generation problem that includes the dynamics of the robot and those of the tether.
\end{itemize}
These contributions are validated by means of numerical simulations, in which we show the benefits of the proposed planning pipeline with respect to existing approaches.

The remainder of this article is organized as follows. In Section \ref{sec:background} we introduce the background material and the problem setting. In Section \ref{sec:modeling} we discuss the construction of the topological model of the entanglement-free configuration space, which is used in Section \ref{sec:path_planning} to compute a set of candidate paths between the robot and the goal location. 
Finally, in Section \ref{sec:trajectories} we present the proposed approach for the optimization-based generation of dynamically feasible entanglement-free trajectories. We validate the proposed method with numerical simulations in Section \ref{sec:case_study}, and we conclude the article in Section \ref{sec:conclusions}.

\section{Background and problem setting}
\label{sec:background}

\subsection{Workspace topology}
Let $\mathcal{W} \subset \mathbb{R}^2$ be a bounded 2D polygonal workspace, and let $\{O_i\}_{i=1}^n$ be a set of $n$ polygonal obstacles such that $O_i\cap O_j=\emptyset, \forall i, j \in \{1, \ldots n\}$. We then define the obstacle region as $\mathcal{O}=\cup_{i=1}^n O_i$ and the free workspace as $\mathcal{W}_\mathrm{free} = \mathrm{cl}(\mathcal{W} \setminus \mathcal{O})$ \cite{latombe1991robot}.
Without loss of generality, we assume that $\mathcal{W}_\mathrm{free}$ is connected.
A path in $\mathcal{W}_\mathrm{free}$ is a continuous mapping $\alpha:[0,1]\rightarrow\mathcal{W}_\mathrm{free}$. We indicate the set of all paths in $\mathcal{W}_\mathrm{free}$ as $\Gamma$, and the set of all paths between two points $a, b\in\mathcal{W}_\mathrm{free}$ as $\Gamma_{a,b}$.
We indicate the length of a path $\alpha$ as $\mathrm{len}(\alpha)$. 
Under the assumptions made in the definition of $\mathcal{W}_\mathrm{free}$, between any two points $a, b$ there exists a shortest path, i.e., one with minimal length, which we indicate as $\widebar{\alpha}_{a,b}$ \cite{burago2001course}.
Two paths that share one endpoint can also be concatenated; given two paths $\gamma_1$, $\gamma_2$ such that $\gamma_1(1)=\gamma_2(0)$, we indicate their concatenation as $\gamma_3 = \gamma_1 \diamond \gamma_2$.

\subsection{Homotopy and homotopy signatures}
Among the $n$ obstacles composing $\mathcal{O}$, the $m \leq n$ obstacles that fully puncture\footnote{I.e., those that have empty intersection with the boundary of $\mathcal{W}$.} $\mathcal{W}$ give rise to multiple homotopy equivalence classes where paths can lie. Given two points $a$, $b$, two paths $\alpha_1, \alpha_2 \in \Gamma_{a, b}$ belong to the same homotopy class if there exists a continuous transformation $H:[0,1]\times[0,1]\rightarrow\mathcal{W}_\mathrm{free}$ between them such that $H(s, 0) = \alpha_1(s)$, $H(s, 1) = \alpha_2(t)$, $H(0, t) = a$, and $H(1, t) = b$ \cite{lee2010introduction}. We indicate homotopy between two paths $\alpha_1$, $\alpha_2$ as $\alpha_1 \sim \alpha_2$, and the homotopy class of a path $\alpha$ as $[\alpha]$ \cite{hatcher2005algebraic}.
Given two points $a, b$, in addition to a globally shortest path $\widebar{\alpha}_{a,b}$, in each homotopy class there exists a locally shortest path. Given a path $\gamma$, we indicate the shortest path in $[\gamma]$ (which is unique for polygonal domains \cite{hershberger1994computing}) by
\begin{equation}
     \begin{aligned}
          \mathcal{S}(\gamma) = &\argmin_{\gamma'\in\Gamma} (\mathrm{len}(\gamma')) \\
                                &\text{ s. t. } \gamma' \sim \gamma.
     \end{aligned}
\end{equation}

A commonly used approach to identify the homotopy class in which a path lies, is that of computing its homotopy signature $h$, which is a function that maps each path to its homotopy class. A signature is defined by a set of $m$ generators $\{\sigma_i\}_{i=1}^m$, one for each of the $m$ obstacles that give rise to multiple homotopy classes in $\mathcal{W}$. Each generator $\sigma_i$ is defined as an infinite ray starting from a point inside an obstacle. Moreover, generators do not intersect with each other \cite{bhattacharya2018path}.
Under this construction, signatures are homotopy invariants, i.e., $h(\alpha_1)=h(\alpha_2) \Leftrightarrow \alpha_1 \sim \alpha_2$ \cite[Proposition 1]{bhattacharya2018path}, where $h(\alpha_i)$ denotes the signature of the path $\alpha_i$. We will henceforth often identify a homotopy class with its signature.   

\subsection{Simplicial complexes and triangulations}
\label{sec:simplicial_complexes_and_triangulations}

A widely used tool to model topological spaces is that of simplicial complexes \cite{edelsbrunner2010computational}.
Simplicial complexes are defined as collections of simplices. A $k$-simplex is the convex hull of a set of $k+1$ affinely independent points $\{x_0, ..., x_k\}\subset\mathbb{R}^d$, $d\in\mathbb{N}$, which we represent by the set of its vertices $\kappa = \{x_i\}_{i=0}^k$.
0-simplices are called points, 1-simplices line segments, and 2-simplices triangles. Given a simplex $\kappa$, a lower-dimensional simplex called a face can be obtained by considering the convex hull of a non-empty subset of vertices $\kappa' \subset \kappa$. We denote this relationship by $\kappa' <  \kappa$. $0$-dimensional faces are called vertices, while $1$-dimensional faces are called edges.
A simplicial complex $S$ is a finite collection of simplices such that (i) $\kappa \in S$ and $\kappa'<\kappa$ imply that also $\kappa'\in S$, and (ii) given $\kappa_1, \kappa_2 \in S$ then $\kappa_1 \cap \kappa_2$ is either empty or a face of both $\kappa_1$ and $\kappa_2$ \cite{edelsbrunner2010computational}.
Given a simplicial complex $S$, we indicate with $S_k$ the subset of all the simplices in $S$ with dimension $k$. The set $S_0$ is called the vertex set, and $S_1$ the edge set of $S$.
The union of all the simplices in $S$, which we indicate by $\mathcal{U}(S)$, is called the underlying space of $S$ \cite{edelsbrunner2010computational}.

In this work we consider manifolds in $\mathbb{R}^2$, which are a class of topological spaces that can be represented as simplicial complexes by mean of triangulations. A triangulation of a manifold $X$ is a simplicial complex $S$ such that $\mathcal{U}(S)$ is homeomorphic to $X$.
We consider here a special class of triangulations, that of constrained triangulations, where the set of 0-simplices and a subset of the 1-simplices of $S$ are predefined. 
This means that the elements of $S_0$ and a subset of $S_1$ are provided as an input to the algorithm computing the triangulation, which is then tasked with computing the remaining part of $S_1$ and the set $S_2$ in order to obtain a finite triangulation of $X$ \cite{preparata1985computational}.

\subsection{Dynamic model of a low-tension tether}
\label{sec:fem-model}
The dynamics of low-tension tethers are commonly modeled numerically, with lumped-mass and finite-element models being the prevailing approaches \cite{buckham2004development,buckham1999dynamics, eidsvik2018finite, bulic2022beam}. 
Lumped-mass models are based on approximately solving the equations of motion of a tether by discretizing them both spatially and temporally. The spatial discretization divides the tether in elements, which are connected to each other at nodes. 
The behavior of each element is described analytically through physics-based equations that define the effect of forces, elasticity, damping, etc., on the element. 
The elements are then assembled into a global system of equations, where continuity (and possibly smoothness) are imposed through nodal constraints. This results in a system of equations of the type
\begin{equation}
     \label{eq:fem-background}
     M\ddot{q}_\mathrm{t}(t) + C\dot{q}_\mathrm{t}(t) + Kq_\mathrm{t}(t) - F(t) = 0,
\end{equation}
where $M$ represents the inertia matrix, $C$ the damping matrix, and $K$ the elasticity matrix, while $F$ is the vector of the forces acting on the nodes, including both internal and external forces. Furthermore,
\begin{equation}
    \label{eq:tether-state}
    q_\mathrm{t}(t)\triangleq[q_\mathrm{t}^{(0)}(t), q_\mathrm{t}^{(1)}(t), \ldots, q_\mathrm{t}^{(z-1)}(t)]^\top
\end{equation}
represents the displacement of the $z$ nodes, while $\dot{q}_\mathrm{t}(t)$ and $\ddot{q}_\mathrm{t}(t)$ represent the velocity and acceleration of the nodes, respectively.
Once assembled, \eqref{eq:fem-background} can be integrated numerically to find the nodal accelerations, from which the nodal velocities and displacements can then be computed.
In the case of an environment with obstacles, it is necessary to detect and resolve collisions between the tether and the obstacles, as well as to prevent self-intersections of the tether \cite{roy2009continuous}. 
When including these effects, the tether dynamics become hybrid, as contacts with obstacles introduce new forces acting on the tether.
For an in-depth discussion of how collisions with obstacles and self-intersections of the tether can be managed, we refer the reader to \cite{roy2009continuous}.

\subsection{Problem setting}
Let $\mathcal{W}, \mathcal{O}, \mathcal{W}_\text{free}$ be respectively the workspace, the obstacle region, and the obstacle-free workspace where a tethered robot operates. 
Moreover, $\mathcal{W}_\text{goal} \subset \mathcal{W}_\text{free}$ indicates the goal region.
The tether is represented by a finite-length obstacle-free path $\gamma:[0,1]\rightarrow\mathcal{W}_\text{free}$ called \emph{tether configuration}. 
The tether initial point $\gamma(0)$ is located at a fixed anchor point $x_\mathrm{a}$, while the terminal point $x_\mathrm{r} = \gamma(1)$ represents the robot location.
The initial tether configuration is indicated as $\gamma_0$, with the initial robot location being then given by $\gamma_0(1)$.
In order to keep track of the homotopy class of the tether configuration and of the paths that we consider during motion planning, we consider a set of generators $\{\sigma_i\}_{i=1}^m$ that satisfy the conditions of \cite[Proposition 1]{bhattacharya2018path}, through which we define a homotopy signature $h$.
In particular, each obstacle that gives rise to multiple homotopy classes is associated to a unique generator, and the generators do not intersect with each other.
In this work, we consider a fixed-length tether, i.e., one whose length is always equal to a constant value $l$.
Additionally, we consider an \emph{entanglement definition}, which provides a criterion to determine if a given tether configuration is entangled or not.
An entanglement definition is an indicator function $\delta:\Gamma \rightarrow \{0, 1\}$ with $\delta(\gamma)=1$ indicating that $\gamma$ is entangled. 
A broad set of entanglement definitions of this type is discussed in \cite{battocletti2024entanglement}.
For convenience, we recall two of those definitions here, and use the same numbering as in \cite{battocletti2024entanglement}.
In particular, we consider the `obstacle free convex hull' entanglement definition \cite[Definition 6]{battocletti2024entanglement}, which evaluates entanglement based on the geometry of the tether, and the `local visibility homotopy' entanglement definition \cite[Definition 9]{battocletti2024entanglement}, which characterizes entanglement based on the homotopy class in which the tether lies.
\setcounter{definition}{5}
\begin{definition}[Obstacle-free Convex Hull \cite{battocletti2024entanglement}]
	\label{def:obstacle_free_conv_hull}
	A tether configuration $\gamma$ is not entangled if its convex hull does not intersect with any obstacle, i.e., $\mathrm{convhull}(\gamma)\cap\interior(\mathcal{O})=\emptyset$, where $\interior(\cdot)$ indicates the interior of a set.
\end{definition}
\setcounter{definition}{8}
\begin{definition}[Local Visibility Homotopy \cite{battocletti2024entanglement}]
	\label{def:local_visibility_homotopy}
    Let $l_{x_1, x_2}$ indicate the straight line segment between two points $x_1$ and $x_2$.
	A tether configuration $\gamma$ is not entangled if, for any pair of points $x_1 = \gamma(s_1)$, $x_2 = \gamma(s_2)$ such that $l_{x_1, x_2}\cap\interior(\mathcal{O})=\emptyset$, it holds that $\gamma_{[s_1, s_2]} \sim l_{x_1, x_2}$.
\end{definition}

The goal is for a robot initially located at $\gamma_0(1)$ to reach the goal region $\mathcal{W}_\text{goal}$ while maintaining the tether in a non-entangled state. We formalize this task as the following motion planning problem.
\begin{problem}(Entanglement-free motion planning.)
    \label{prob:entanglement-free-planning}
	Given a motion planning problem for tethered robots defined by a tuple $(\mathcal{W}_\text{free}, \mathcal{W}_\text{goal}, \gamma_0)$, and an entanglement definition $\delta$, find a trajectory $\beta$, if one exists, such that the robot can reach $\mathcal{W}_\text{goal}$ by following $\beta$ while respecting the tether length constraint and retaining a non-entangled tether configuration.
\end{problem}
An example of a path planning problem for a tethered robot is depicted in Figure \ref{fig:problem-setting-a}.
In order for Problem \ref{prob:entanglement-free-planning} to have a solution, a non-entangled initial tether configuration with $\delta(\gamma_0) = 0$ is a necessary condition. 
Resolving an initially entangled tether configuration yields the  second motion planning problem that we seek to solve, which is depicted in Figure \ref{fig:problem-setting-b}.
\begin{problem}(Disentanglement motion planning.)
     \label{prob:disentanglement-pp}
	Given the data $(\mathcal{W}_\text{free}, \gamma_0)$, and an entanglement definition $\delta$ such that $\delta(\gamma_0) = 1$, find a trajectory $\beta$, if one exists, such that by following $\beta$ the robot leads the tether to a non-entangled configuration.
\end{problem}
\begin{figure}
	\centering
    \begin{subfigure}[b]{0.48\linewidth}
        \centering
        \includegraphics[width=\linewidth]{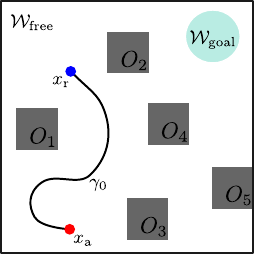}
        \caption{}
        \label{fig:problem-setting-a}
    \end{subfigure}
    \hfill
    \begin{subfigure}[b]{0.48\linewidth}
        \centering
        \includegraphics[width=\linewidth]{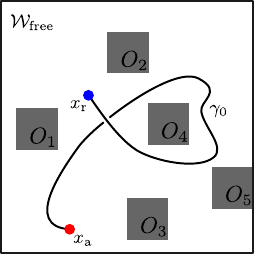}
        \caption{}
        \label{fig:problem-setting-b}
    \end{subfigure}
	\caption{(a) Example of a path planning problem for a tethered mobile robot. (b) Example of disentanglement path planning problem.}
	\label{fig:problem-setting}
\end{figure}

\section{Modeling of the configuration space}
\label{sec:modeling}
In this section we describe the construction of a topological model of the entanglement-free configuration space of the tethered robot.
We start by formally defining the entanglement-free configuration space in Section \ref{sec:config-space}, for which we then build a computationally tractable model.
It is well known that the configuration space of a tethered robot can be modeled as a simply connected subset of the universal covering space of $\mathcal{W}_\mathrm{free}$ \cite{bhattacharya2018path, battocletti2025efficient, teshnizi2014computing}. 
To this end, a widely used method is that of constructing a homotopy-augmented graph \cite{igarashi2010homotopic, kim2014path}.
However, homotopy-augmented graphs are computationally expensive to generate \cite{bhattacharya2021search}, and only offer a discrete approximation of the configuration space of the tethered robot. Moreover, considering an entanglement constraint requires checking it at each node of the graph, further increasing the computational cost.

We propose here to address these issues by following a different approach to model the configuration space of a tethered robot.
In particular, we directly construct a topological model of the configuration space, rather than building a graph embedded in it.
To do so, we consider the simplicial complex model first proposed in \cite{battocletti2025efficient}. 
The resulting topological model has the advantages of being continuous, eliminating the discretization error intrinsic in graph-based methods, and is significantly cheaper to compute than homotopy-augmented graphs.
Moreover, as we will show in Section \ref{sec:verification-of-entanglement}, the entanglement constraint can be verified efficiently and only in a low number of points in the workspace.
This result will be used in Section \ref{sec:simplicial-complex-model} to efficiently build a simplicial complex model of the entanglement-free configuration space of a tethered robot.

\subsection{Entanglement-free configuration space of a tethered robot}
\label{sec:config-space}
We begin by formally defining the configuration of a tethered robot \cite{yang2022efficient_b}.
\setcounter{definition}{0}
\begin{definition}[Configuration of a tethered robot]
	\label{def:configuration}
	Given a free workspace $\mathcal{W}_\mathrm{free}$ and an anchor point $x_\mathrm{a} \in \mathcal{W}_\mathrm{free}$, a configuration of a tethered robot is a path $\gamma: [0,1]\rightarrow\mathcal{W}_\mathrm{free}$ with $\gamma(0) = x_\mathrm{a}$.
\end{definition}
The configuration space of a tethered mobile robot under Definition \ref{def:configuration} is then
\begin{equation}
	\mathcal{C}(x_\mathrm{a}) = \{\gamma:[0,1]\rightarrow \mathcal{W}_\mathrm{free} :\gamma(0) = x_\mathrm{a}, \gamma \; \mathrm{continuous}\}.
\end{equation}
The configuration space $\mathcal{C}(x_\mathrm{a})$ is infinite dimensional, as it contains all the possible tether configurations that satisfy the tether length constraint, and is therefore difficult to model.
A common approach to overcome this issue for the purpose of motion planning is to introduce a \emph{reduced} configuration, which does not include the full information about the tether shape, but only that about the homotopy class in which the tether lies \cite{kim2014path,battocletti2025efficient}.
\begin{definition}[Reduced configuration of a tethered robot]
	\label{def:configuration-reduced}
	Given a configuration of a tethered robot $\gamma \in \mathcal{C}(x_\mathrm{a})$ and a signature $h$ associated to the free space $\mathcal{W}_\mathrm{free}$, a reduced configuration of a tethered robot is defined as $r(\gamma) = (\gamma(1), h(\gamma))$.  
\end{definition}
Reduced configurations introduced in Definition \ref{def:configuration-reduced} live in a `smaller' configuration space given by 
\begin{equation}
    \mathcal{R}(x_\mathrm{a}, \{\sigma_i\}_{i=1}^m) = \mathcal{W}_\mathrm{free} \times \mathbb{F}_m,
\end{equation}
where $\mathbb{F}_m$ represents the free group built from the $m$ generators.
In practice, in $\mathcal{R}(x_\mathrm{a}, \{\sigma_i\}_{i=1}^m)$ we only keep track of the homotopy class through which a point is reached, rather than of the exact tether configuration. This allows to enumerate all the homotopy classes through which a point can be reached, corresponding to the elements of $\mathbb{F}_m$.
Furthermore, points in $\mathcal{R}(x_\mathrm{a}, \{\sigma_i\}_{i=1}^m)$ have a one-to-one correspondence with points in the so-called universal covering space\footnote{
    The universal covering space $\widetilde{\mathcal{W}}_\mathrm{free}$ of ${\mathcal{W}}_\mathrm{free}$ is a simply connected topological space such that every loop in ${\mathcal{W}}_\mathrm{free}$ lifts to a path in $\widetilde{\mathcal{W}}_\mathrm{free}$, with the endpoints of the lifted path determined by the homotopy class of the loop (see \cite[Chapter 11]{lee2010introduction}).
    This construction provides a natural one-to-one correspondence between reduced tether configurations and points in $\widetilde{\mathcal{W}}_\mathrm{free}$.
    We refer to \cite{battocletti2025efficient} for this precise representation.
} $\widetilde{\mathcal{W}}_\mathrm{free}$ of $\mathcal{W}_\mathrm{free}$, a property that will be leveraged in Section \ref{sec:simplicial-complex-model} to build the topological model of the entanglement-free configuration space.
The number of homotopy classes through which a point can be reached is, in general, infinite; however, given the maximum tether length $l$, only a finite number of homotopy classes admit a path that is shorter than or equal to $l$. 
This leads us to define the length-constrained reduced configuration space $\widebar{\mathcal{R}}(x_\mathrm{a}, \{\sigma_i\}_{i=1}^m, l)$ as
\begin{equation}
     \begin{aligned}
          \widebar{\mathcal{R}}(x_\mathrm{a}&, \{\sigma_i\}_{i=1}^m, l) = \{(x, s) \in\mathcal{W}_\mathrm{free}\times\mathbb{F}_m : \\&\exists \gamma \in \Gamma_{x_\mathrm{a}, x} \text{ such that } [\gamma] \cong s \text{ and } \mathrm{len}(\gamma) \leq l \}.
     \end{aligned}
\end{equation}
The configuration space $\widebar{\mathcal{R}}(x_\mathrm{a}, \{\sigma_i\}_{i=1}^m, l)$ is a subset of $\mathcal{R}(x_\mathrm{a}, \{\sigma_i\}_{i=1}^m)$ corresponding to finite-length tether configurations, and can therefore be conveniently represented by a bounded simplicial complex on $\widetilde{\mathcal{W}}_\mathrm{free}$.
 
We discuss now how the entanglement constraint can be included in the configuration space. By directly including the entanglement constraint in Definition \ref{def:configuration}, i.e., by restricting $\mathcal{C}(x_\mathrm{a})$ to only the tether configurations that are not entangled with respect to the entanglement definition $\delta$, we obtain the set
\begin{equation}
     \mathcal{N}(x_\mathrm{a}, \delta) = \{\gamma \in \mathcal{C}: \delta(\gamma) = 0\}.
\end{equation}
Since the entanglement-free configuration space is still an infinite dimensional space of curves, we follow the same reduction approach used to go from $\mathcal{C}(x_\mathrm{a})$ to $\mathcal{R}(x_\mathrm{a}, \{\sigma_i\}_{i=1}^m)$ in order to obtain the reduced entanglement-free configuration space and define:
\begin{equation}
     \begin{aligned}
          \widebar{\mathcal{N}}(x_\mathrm{a}, \{\sigma_i\}_{i=1}^m,& l, \delta) = \{(x, s) \in\mathcal{W}_\mathrm{free}\times\mathbb{F}_m : \\&\exists \gamma \in \Gamma_{x_\mathrm{a}, x} \text{ such that } [\gamma] \cong s, \\
          &\mathrm{len}(\gamma) \leq l, \\
          &\text{and } \delta(\gamma)=0\}.
     \end{aligned}
\end{equation}
For all the elements $(x, h)\in\widebar{\mathcal{N}}$ we know that there exists, in the homotopy class $h$, at least one path that (i) has length less than or equal to $l$, and (ii) is not entangled with respect to $\delta$.
For the reduced configuration spaces defined so far the following inclusions hold:
\begin{equation}
	\widebar{\mathcal{N}}(x_\mathrm{a}, \{\sigma_i\}_{i=1}^m, l, \delta) \subseteq \widebar{\mathcal{R}}(x_\mathrm{a}, \{\sigma_i\}_{i=1}^m, l) \subsetneq \mathcal{R}(x_\mathrm{a}, \{\sigma_i\}_{i=1}^m).
\end{equation}

In the next section we discuss how the entanglement constraint can be verified in a computationally-tractable way over $\widebar{\mathcal{R}}(x_\mathrm{a}, \{\sigma_i\}_{i=1}^m, l)$ in order to find $\widebar{\mathcal{N}}(x_\mathrm{a}, \{\sigma_i\}_{i=1}^m, l, \delta)$. For ease of notation, in the following we drop the arguments of $\widebar{\mathcal{R}}$ and $\widebar{\mathcal{N}}$.

\subsection{Verification of the entanglement constraint}
\label{sec:verification-of-entanglement}
We now illustrate a computationally-tractable criterion to go from $\widebar{\mathcal{R}}$ to $\widebar{\mathcal{N}}$.
The first step is, given a point $(x, s) \in \widebar{\mathcal{R}}$, to determine an efficient way to verify the existence of an entanglement-free path between the anchor point $x_\mathrm{a}$ and the point $x$ in the homotopy class $s$. 
In general, verifying the existence of an entanglement-free path may require checking infinitely many paths.
To tackle this issue, we show that, for all the binary entanglement definitions discussed in \cite{battocletti2024entanglement}, there exists such an entanglement-free path in a homotopy class if and only if the shortest path in this class is also entanglement-free. 
\begin{proposition}[Entanglement state of shortest path]
     \label{prop:entanglement-shortest-path}
     Let $\delta$ be a binary entanglement definition among those listed in \cite{battocletti2024entanglement}.
     Given a tether configuration $\gamma$ between two points $x_1$, $x_2$, if $\mathcal{S}(\gamma)$ is entangled with respect to an entanglement definition $\delta$, then also $\gamma$ is entangled with respect to $\delta$.
\end{proposition}
\begin{proof}
     The proof follows from those of \cite[Propositions 1--4]{battocletti2024entanglement}, by restricting the set of paths considered there to the set $\Gamma'_{x_\mathrm{a}, x} = \{\gamma'\in\Gamma_{x_\mathrm{a}, x} : \gamma' \sim \gamma \}$.
\end{proof}
Proposition \ref{prop:entanglement-shortest-path} provides a simple criterion to assess the existence of at least one entanglement-free tether configuration between the anchor point and any point in the environment for a given homotopy class.
More precisely, if the shortest path $\mathcal{S}(\gamma)$ in a homotopy class $[\gamma]$ is entangled, then there cannot exist any other non-entangled tether configuration in that homotopy class.
An example is visualized in Figure \ref{fig:triangle-check-a}.
Therefore, given a point $(x, s) \in \widebar{\mathcal{R}}$, it is sufficient to verify if the shortest path from the anchor point to $x$ in $s$ is entangled or not to either add the point $(x, s)$ to $\widebar{\mathcal{N}}$, as it is possible to reach it through an entanglement-free tether configuration, or to discard it.

Since we aim to build a simplicial complex model of $\widebar{\mathcal{N}}$, we focus now on applying the verification of the entanglement constraint described in Proposition \ref{prop:entanglement-shortest-path} to simplices.
We start by noting that all the binary entanglement definitions considered in \cite{battocletti2024entanglement} preserve the entanglement state of a non-entangled tether configuration over certain types of homotopies.
\begin{proposition}
    \label{prop:entanglement-linear-homotopy}
    Given two tether configurations $\gamma_1$, $\gamma_2$ such that $\gamma_1(0)=\gamma_2(0)=x_\mathrm{a}$ that are not entangled with respect to an entanglement definition $\delta$ of \cite{battocletti2024entanglement} (besides Definitions 2 and 3)\footnote{The proposition is not applicable to \cite[Definitions 2, 3]{battocletti2024entanglement} as those definitions regard only multi-robot systems, which are not considered here.}, and the path $\widebar{\alpha}_{\gamma_1(1), \gamma_2(1)}$, if $\gamma_2 \sim \gamma_1 \diamond \widebar{\alpha}_{\gamma_1(1), \gamma_2(1)}$ then for all points $x\in\widebar{\alpha}_{\gamma_1(1), \gamma_2(1)}$ there exists a tether configuration $\gamma'\in\Gamma_{x_\mathrm{a}, x}$ with $\gamma'\sim\gamma_1\diamond\widebar{\alpha}_{\gamma_1(1), x}$ that is not entangled with respect to $\delta$.
\end{proposition}
\begin{proof}[Proof]
    For \cite[Definitions 1, 5, 6, 7, 9]{battocletti2024entanglement} the proof follows from the characterization of the sets $\mathcal{N}_{x_\mathrm{a}, d}$ defined in \cite[Propositions 1, 2]{battocletti2024entanglement}.
    For \cite[Definition 4]{battocletti2024entanglement} the proof is trivial as $\widebar{\alpha}_{\gamma_1(1), \gamma_2(1)} = x_\mathrm{a}$ since $\gamma_1$ and $\gamma_2$ are loops.
    For \cite[Definitions 8, 10, 11]{battocletti2024entanglement} the proof is the same as that of the `base' entanglement definitions that are relaxed by \cite[Definitions 8, 10, 11]{battocletti2024entanglement}, and thus depends on the choice of the `base' definition.
\end{proof}
An example of the scenario described in Proposition \ref{prop:entanglement-linear-homotopy} is shown in Figure \ref{fig:triangle-check-b}.
Proposition \ref{prop:entanglement-linear-homotopy}, paired with the convexity of the $k$-simplices, results in the fact that the verification of the entanglement constraint at the vertices of a $k$-simplex is sufficient to ensure that the entanglement constraint is satisfied in all the points of that simplex.
\begin{proposition}[Entanglement-free triangle]
    \label{prop:entanglement-k-simplex}
	Consider a $k$-simplex $\kappa$, with $k\leq2$ and vertices denoted by $\{x_i\}_{i=1}^{k+1}$, and a path $\gamma$ between $x_\mathrm{a}$ and a point $x \in s$.
    If $\mathcal{S}(\gamma\diamond\widebar{\alpha}_{x, x_i})$ is not entangled with respect to the entanglement definition $\delta$ of \cite{battocletti2024entanglement} for all the $k+1$ vertices, then $\mathcal{S}(\gamma\diamond\widebar{\alpha}_{x, x'})$ is also not entangled for any point $x'\in \mathrm{convhull}(\kappa)$.
\end{proposition}
\begin{proof}
     The proof follows from Proposition \ref{prop:entanglement-linear-homotopy} and the convexity of the $k$-simplices.
\end{proof}
The key result of Proposition \ref{prop:entanglement-k-simplex} is that it is possible to verify if all the points in a triangle can be reached through an entanglement-free tether configuration by evaluating the entanglement definition $\delta$ only at its vertices. An example of this process is shown in Figure \ref{fig:triangle-check-c}.
\begin{figure}
    \centering
    \begin{subfigure}[b]{0.49\linewidth}
        \centering
        \includegraphics[width=0.9\linewidth]{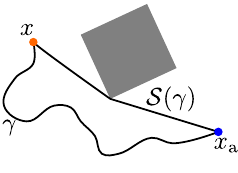}
        \caption{}
        \label{fig:triangle-check-a}
    \end{subfigure}
    \hfill
    \begin{subfigure}[b]{0.49\linewidth}
        \centering
        \includegraphics[width=\linewidth]{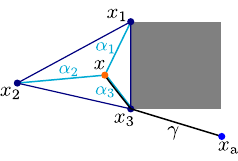}
        \caption{}
        \label{fig:triangle-check-b}
    \end{subfigure}\\
    \vspace{0.3cm}
    \begin{subfigure}[b]{0.49\linewidth}
        \centering
        \includegraphics[width=\linewidth]{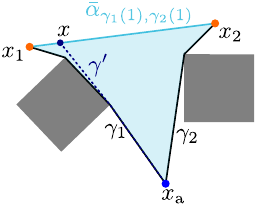}
        \caption{}
        \label{fig:triangle-check-c}
    \end{subfigure}
    \hfill
    \begin{subfigure}[b]{0.49\linewidth}
        \centering
        \includegraphics[width=\linewidth]{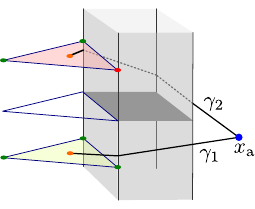}
        \caption{}
        \label{fig:triangle-check-d}
    \end{subfigure}
    \caption{(a) Example of verification of the entanglement state of a tether configuration $\gamma$ via the shortened path $\mathcal{S}(\gamma)$. (b) Example of path $\bar{\alpha}_{x_1, x_2}$ between two tether configurations $\gamma_1$, $\gamma_2$ along which the entanglement state is constant. (c) Example of verification of the vertices of a triangle to determine if all its points are reachable through an entanglement-free tether configuration in a given homotopy class. (d) Example of the same triangle being reached through two different homotopy classes, resulting in different admissibility with respect to the entanglement constraint.}
    \label{fig:triangle-check}
\end{figure}

\begin{remark}
    Propositions \ref{prop:entanglement-shortest-path}, \ref{prop:entanglement-linear-homotopy}, and \ref{prop:entanglement-k-simplex} have been proven for the entanglement definitions of \cite{battocletti2024entanglement}, as that paper lists the large majority of the currently existing entanglement definitions.
    Nonetheless, we remark that the proposed approach is general and can be applied also to new entanglement definitions, as long as Propositions \ref{prop:entanglement-shortest-path}, \ref{prop:entanglement-linear-homotopy}, and \ref{prop:entanglement-k-simplex} hold for the entanglement definition being considered.
\end{remark}

\subsection{Simplicial complex model of the entanglement-free configuration space of a tethered robot}
\label{sec:simplicial-complex-model}
Equipped with the results from the previous section, we now introduce an algorithm to efficiently build a topological model of the entanglement-free configuration space $\widebar{\mathcal{N}}$.
The proposed topological model of $\widebar{\mathcal{N}}$ is based on the incremental construction of a simplicial complex on a subset of the universal covering space $\widetilde{\mathcal{W}}_\mathrm{free}$ that respects the tether length constraint and the entanglement avoidance constraint.
Taking inspiration from \cite{preparata1985computational, battocletti2025efficient}, we propose to construct the simplicial complex starting from a constrained triangulation $\mathcal{T}$ of $\mathcal{W}_\mathrm{free}$ (see Figure \ref{fig:triangulation-b}).\footnote{As shown in Figure \ref{fig:triangulation}, the base triangulation yields also a primal graph $\mathcal{G} = (\mathcal{V}, \mathcal{E})$, whose vertex $\mathcal{V}$ set coincides with $\mathcal{T}_0$ and edge set $\mathcal{E}$ with $\mathcal{T}_1$, and a dual graph $\mathcal{G}' = (\mathcal{V}', \mathcal{E}')$, whose vertex set $\mathcal{V}'$ is formed by a set containing one point from the interior of each triangle (so that each triangle is uniquely identified by a point in $\mathcal{V}'$, and vice versa) and whose edge set $\mathcal{E}'$ is composed by pairs of points ${v_1, v_2}, v_1, v_2 \in \mathcal{V}'$ indicating adjacency between the triangles corresponding to the points $v_1$ and $v_2$. These graphs are leveraged during the lifting process, as discussed in \cite{battocletti2025efficient}.}
The triangulation is constrained by imposing the coincidence of $\mathcal{T}_0$ with the set of vertices of the polygonal obstacles $\{O_i\}_{i=1}^n$, and by imposing that all the edges defining the boundary of the polygonal obstacles are part of $\mathcal{T}_1$ \cite{preparata1985computational}.
The triangulation $\mathcal{T}$ is then lifted to the universal covering space of $\mathcal{W}_\mathrm{free}$ \cite{battocletti2025efficient}, resulting in a simplicial complex with different branches corresponding to the different homotopy classes in which the environment can be partitioned \cite{hershberger1994computing}.
While lifting $\mathcal{T}$ the tether length constraint is verified to obtain a simplicial complex model of $\widebar{\mathcal{R}}$, and the entanglement avoidance one to obtain a simplicial complex model of $\widebar{\mathcal{N}}$.
\begin{figure}
    \centering
    \begin{subfigure}[b]{0.49\linewidth}
        \centering
        \includegraphics[width=\linewidth]{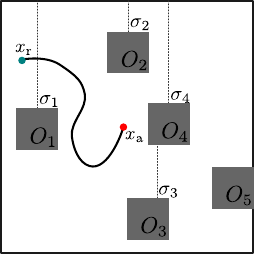}
        \caption{}
        \label{fig:triangulation-a}
    \end{subfigure}
    \hfill
    \begin{subfigure}[b]{0.49\linewidth}
        \centering
        \includegraphics[width=\linewidth]{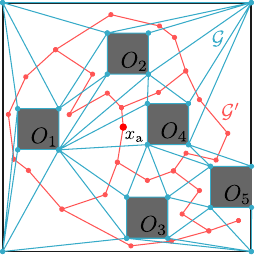}
        \caption{}
        \label{fig:triangulation-b}
    \end{subfigure}\\
    \vspace{0.3cm}
    \begin{subfigure}[b]{0.49\linewidth}
        \centering
        \includegraphics[width=\linewidth]{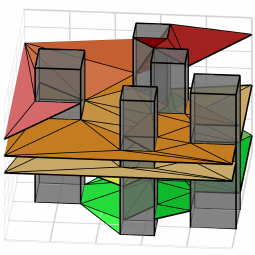}
        \caption{}
        \label{fig:triangulation-c}
    \end{subfigure}
    \hfill
    \begin{subfigure}[b]{0.49\linewidth}
        \centering
        \includegraphics[width=\linewidth]{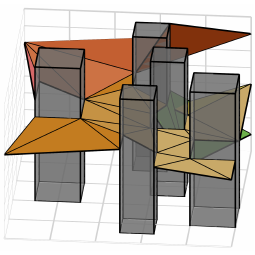}
        \caption{}
        \label{fig:triangulation-d}
    \end{subfigure}
    \caption{(a) Example of 2D workspace with $n=5$, $m=4$. (b) Constrained Delaunay triangulation $\mathcal{T}$ of the free workspace $\mathcal{W}_\mathrm{free}$. The primal graph $\mathcal{G}$ is shown in blue, and the dual one $\mathcal{G}'$ in red. (c) Simplicial complex model $\widetilde{\mathcal{R}}$ of the length-reachable configuration space. (d) Simplicial complex model $\widetilde{\mathcal{N}}^{(9)}$ of the entanglement-free configuration space with respect to the entanglement definition 9 of \cite{battocletti2024entanglement}.}
    \label{fig:triangulation}    
\end{figure}
To this end, we exploit the result of Proposition \ref{prop:entanglement-k-simplex} to efficiently verify the length and entanglement constraints during the construction of the simplicial complex. This way, triangles are added only if they satisfy both the tether length constraint and the entanglement one.
Each point in the resulting simplicial complex model (i) is reachable within the maximum tether length in the branch of the universal covering space corresponding to its homotopy class (which is guaranteed by the original algorithm from \cite{battocletti2025efficient}), and (ii) admits the existence of a tether configuration that is not entangled under the entanglement definition being considered.
This way, we obtain a simplicial complex model of $\widebar{\mathcal{N}}$.
Thanks to the result of Proposition \ref{prop:entanglement-k-simplex}, the increase in terms of computational cost due to the verification of the entanglement constraint is equal to one evaluation of the entanglement definition $\delta$ for each vertex of the simplices $\widetilde{\mathcal{N}}$. Moreover, the shortest paths $\mathcal{S}(\gamma\diamond\alpha_i), i\in\{1, 2, 3\}$ from $x_\mathrm{a}$ to the vertices of a simplex are already available from the verification of the length constraint.

As in the original model from \cite{battocletti2025efficient}, the simplicial complex model can contain multiple copies of a given triangle, each corresponding to a different homotopy class though which that triangle can be reached, as shown in Figure \ref{fig:triangle-check-d}. For each copy of a given triangle, the entanglement condition is verified in its three vertices, and only if in all the three vertices the shortest path is entanglement-free, the triangle is added. In the example of Figure \ref{fig:triangle-check-d}, the triangle at the bottom is reached through a homotopy class in which all the three vertices satisfy the entanglement constraint, and is therefore added to the simplicial complex. Contrarily, for the triangle on top, one of the vertices (marked in red) does not satisfy the entanglement constraint, resulting in the triangle being excluded from the simplicial complex.
The modified version of the algorithm from \cite{battocletti2025efficient} is reported in Algorithm \ref{alg:entanglement-free-simplicial-complex}.
\begin{algorithm}
    \caption[Simplicial complex model of the en\-tan\-gle\-ment-free workspace]{Simplicial complex model of $\widebar{\mathcal{N}}$ and $\widebar{\mathcal{R}}$}
	\label{alg:entanglement-free-simplicial-complex}
	\begin{algorithmic}[1]
		\State \textbf{Inputs}: $\mathcal{T}$, $x_\mathrm{a}$, $l$
            \State Compute $\mathcal{G}$ and $\mathcal{G}'$ from $\mathcal{T}$, add $x_\mathrm{a}$ to the vertex set of $\mathcal{G}'$
            \State Initialize $\widetilde{\mathcal{R}}\gets\emptyset$, $\widetilde{\mathcal{N}}\gets\emptyset$, $q_\mathrm{open} \gets \{(x_\mathrm{a}, \mathrm{``\;"})\}$, $q_\mathrm{closed} \gets \emptyset$, $V \gets \emptyset$, $\alpha_{x_\mathrm{a}, p} \gets \emptyset$
            \While{$q_\mathrm{open} \neq \emptyset$} \label{line:main_loop_start}

                \Statex \quad\; \textcolor{gray}{\# Select next triangle to add to $\widetilde{\mathcal{N}}$}
                \State Pop $(p, s)$ from $q_\mathrm{open}$
                \State \textbf{If} $(p, s) \in q_\mathrm{closed}$ \textbf{then} skip to next iteration
                \State $V \gets \mathtt{vertices}(\mathtt{triangle}(p))$\label{line:triangle_check_start}
                \State $\alpha_{x_\mathrm{a}, p} \gets $ shortest path in $\Gamma_{x_\mathrm{a}, p}$ such that $h(\alpha_{x_\mathrm{a}, p}) = s$

                \Statex \quad\; \textcolor{gray}{\# Verification of tether length constraint}
                \For{$v\in V$}
                    \State \textbf{If} $\mathrm{len}(\mathcal{S}({\alpha_{x_\mathrm{a}, p}}\diamond\alpha_{p, v})) > l$ \textbf{then} go to line \ref{line:main_loop_start}
                \EndFor

                \Statex \quad\; \textcolor{gray}{\# Add simplices to $\widetilde{\mathcal{R}}$}
                \State Compute 0-simplices $(v_i, s'_i), v_i\in V, s'_i = s\cdot h(\widebar{\alpha}_{p, v})$
                \State $\mathtt{add\_simplices}(V, \widetilde{\mathcal{R}})$

                \Statex \quad\;\textcolor{gray}{\# Verification of entanglement constraint}
                \For{$v\in V$}
                        \State \textbf{If} $\delta(\mathcal{S}({\alpha_{x_\mathrm{a}, p}}\diamond\alpha_{p, v})) = 1$ \textbf{then} skip to line \ref{line:add_adjacent_triangles}
                \EndFor

                \Statex \quad\;\textcolor{gray}{\# Add simplices to $\widetilde{\mathcal{N}}$}
                \State Compute 0-simplices $(v_i, s'_i), v_i \in V, s'_i = s\cdot h(\widebar{\alpha}_{p, v})$
                \State $\mathtt{add\_simplices}(V, \widetilde{\mathcal{N}})$

                \Statex \quad\; \textcolor{gray}{\# Add adjacent triangles to open queue}
                \For{$p' \in \mathtt{adjacent}(p)$} \label{line:add_adjacent_triangles}
                \State $s' \gets s \cdot h(\widebar{\alpha}_{p, p'})$ 
                \State Add $(p', s')$ to $q_\mathrm{open}$ 
                \EndFor

                \Statex \quad\; \textcolor{gray}{\# Mark current triangle as visited, then continue}
                \State Add $(p, s)$ to $q_\mathrm{closed}$

            \EndWhile
            \State \textbf{Return} $\widetilde{\mathcal{R}}$ and $\widetilde{\mathcal{N}}$
	\end{algorithmic}
\end{algorithm}

Algorithm \ref{alg:entanglement-free-simplicial-complex} returns two simplicial complexes, $\widetilde{\mathcal{R}}$ and $\widetilde{\mathcal{N}}$. The first corresponds to the simplicial complex model of the length-constrained configuration space, while the second to the simplicial complex model of the entanglement-free configuration space. The former is computed at no extra cost with respect to the latter, and will be useful for the disentanglement path planning described in Section \ref{sec:disentanglement-path-planning}.

Both $\mathcal{U}(\widetilde{\mathcal{N}})$ and $\mathcal{U}(\widetilde{\mathcal{R}})$ are simply connected manifolds with a boundary \cite{edelsbrunner2010computational}.
Between the two simplicial complexes it holds, by construction, that $\widetilde{\mathcal{N}} \subseteq \widetilde{\mathcal{R}}$, and therefore also that $\mathcal{U}(\widetilde{\mathcal{N}}) \subseteq \mathcal{U}(\widetilde{\mathcal{R}})$.
Both simplicial complexes also yield a primal and a dual graph, corresponding to a lifted version of the graphs $\mathcal{G}$ and $\mathcal{G}'$. 
We denote by $\widetilde{\mathcal{G}}_{\tilde{\mathcal{R}}}$ and $\widetilde{\mathcal{G}}_{\tilde{\mathcal{R}}}'$ the lifted graphs corresponding to $\widetilde{\mathcal{R}}$, and as $\widetilde{\mathcal{G}}_{\tilde{\mathcal{N}}}$ and $\widetilde{\mathcal{G}}_{\tilde{\mathcal{N}}}'$ for the graphs corresponding to $\widetilde{\mathcal{N}}$.
The lifted graphs $\widetilde{\mathcal{G}}_i$ and $\widetilde{\mathcal{G}}'_i$, $i \in \{\widetilde{\mathcal{R}}, \widetilde{\mathcal{N}}\}$, represent useful data structures for the planning phase, which will be discussed in Section \ref{sec:path_planning}.

As a concluding remark, we note that, as observed in \cite{battocletti2025efficient}, the simplicial complex model can be conservative at its boundaries. In fact, simplices are excluded from the simplicial complex if any of their vertices does not satisfy the entanglement definition, even if part of the simplex satisfies it. A possible strategy to mitigate this limitation of Algorithm \ref{alg:entanglement-free-simplicial-complex} is discussed in Appendix \ref{appendix:reduction-of-conservativeness}.

\section{Entanglement-free path planning}
\label{sec:path_planning}
In this section we detail how the simplicial complex models $\widetilde{\mathcal{R}}$ and $\widetilde{\mathcal{N}}$ can be used both for path planning and for tether disentanglement.

\subsection{Path planning on the simplicial complex model}
The simplicial complex model $\widetilde{\mathcal{N}}$ can be used to compute paths that allow to navigate the workspace while maintaining the tether in a non-entangled configuration. 
In fact, by construction, all the paths computed in $\widetilde{\mathcal{N}}$ lie in $\widebar{\mathcal{N}}$, meaning that for each point along the path there exists an entanglement-free tether configuration, and therefore the path can be followed while keeping the tether in a non-entangled state.
\begin{proposition}[Entanglement-free path in $\widetilde{\mathcal{N}}$]
    Given a path $\alpha:[0,1]\rightarrow\mathcal{U}(\widetilde{\mathcal{N}})$ of the robot, there exists a path of entanglement-free tether configurations with one endpoint fixed in $\widetilde{x}_\mathrm{a}$ and the other moving along $\alpha$.
\end{proposition}
\begin{proof}
    The path $\alpha$ is always inside $\mathcal{U}(\widetilde{\mathcal{N}})$, for which, by construction, we know that at every point there exists a non-entangled tether configuration.
    Since $\mathcal{U}(\widetilde{\mathcal{N}})$ is simply connected, and therefore all paths joining two points lie in a single homotopy class, from Proposition \ref{prop:entanglement-shortest-path} we know that the shortest path between any point in $\widetilde{\mathcal{N}}$ and the anchor point is not entangled. Therefore, the family of shortest paths between the points along $\alpha$ and the anchor point $x_\mathrm{a}$ yields a valid path of entanglement-free tether configurations.
\end{proof}
Path planning can be performed in several ways on $\widetilde{\mathcal{N}}$. Both the lifted primal graph $\widetilde{\mathcal{G}}_{\tilde{\mathcal{N}}}$ and the lifted dual graph and $\widetilde{\mathcal{G}}_{\tilde{\mathcal{N}}}'$ are suitable for this purpose by using a graph search algorithm such as  Depth First Search, Dijkstra, and A$^\star$ \cite{latombe1991robot}. In particular, since $\widetilde{\mathcal{G}}_{\tilde{\mathcal{N}}}'$ is a tree, it can be used to quickly compute all the topologically distinct paths between the robot location and the goal region \cite{battocletti2025efficient}. 
Alternatively, since $\mathcal{U}(\widetilde{\mathcal{N}})$ is a simply-connected manifold, sampling-based path planning algorithms such as RRT and RRT$^\star$, as well as other classes of path planning algorithms \cite{paden2016survey, lavalle2006planning}, can be applied to the simplicial complex model $\widetilde{\mathcal{N}}$.

In case of a taut tether, which is a common type of tether, usually maintained through a winch and a dedicated tether control system, the paths computed on the simplicial complex model are directly guaranteed to be traversable while maintaining the tether in an entanglement-free tether configuration, even when considering the dynamics of the tether and of the robot.\footnote{In this sense, we note that the low-tension tether models introduced in Section \ref{sec:fem-model} are not usually applicable to the taut-tether case, and we refer the interested reader to works that consider this case \cite{lima2025tension, mcgarey2016line, tognon2017dynamics}.}
In fact, as long as the robot remains inside the manifold $\mathcal{U}(\widetilde{\mathcal{N}})$, the shortest path between its location and the anchor point -- which coincides with the taut tether configuration -- is entanglement-free by construction of $\widetilde{\mathcal{N}}$.
On the contrary, this does not hold for the more general case of a slack tether that we consider in this work. In fact, the existence of a path along which the tether remains not entangled, does not guarantee that this can be achieved when considering the dynamics of the robot and that of the tether. For this reason, we propose the trajectory generation step introduced in Section \ref{sec:trajectories}. 
Therefore, it is convenient to search for multiple candidate paths $\{\alpha_i\}_{i=1}^\nu$, so that multiple trajectory generation problems can be run in parallel to increase the effectiveness of the planning algorithm.

\subsection{Disentangling path planning}
\label{sec:disentanglement-path-planning}
When the initial tether configuration $\gamma_0$ is in an entangled state, before searching for a path to the goal region, it is crucial to disentangle the tether.
To this end, we use a combination of the length-constraint simplicial complex model $\widetilde{\mathcal{R}}$, corresponding to the subset of the workspace that the robot can physically reach without exceeding the tether length constraint, and the entanglement-free simplicial complex model $\widetilde{\mathcal{N}}$, corresponding to all the points that can be reached through a non-entangled tether configuration. 
To find a path leading from the initial, entangled, tether configuration, to a non-entangled one, we run a path planning algorithm on $\widetilde{\mathcal{R}}$ (or, equivalently, on one of the two graphs $\widetilde{\mathcal{G}}_{\tilde{\mathcal{R}}}$ and $\widetilde{\mathcal{G}}_{\tilde{\mathcal{R}}}'$), until we find a path to a point in $\widetilde{\mathcal{R}}$ that is also in $\widetilde{\mathcal{N}}$, i.e., for which there exists a non-entangled tether configuration.
Similarly to the path planning case described above, the returned path does not guarantee the existence of a dynamically feasible trajectory for the robot, but rather identifies a homotopy class through which such a trajectory should be searched for. 

\section{Computation of dynamically feasible entanglement-free trajectories}
\label{sec:trajectories}
Once one or more candidate paths are selected, they can be used as the starting point for the generation of dynamically feasible trajectories by solving a homotopy-constrained trajectory generation problem \cite{degroot2024topology}, which we formulate as a mixed-integer optimization problem taking inspiration from \cite{park2015homotopy}.

In the optimization problem the discrete variables emerge from the presence of a nonconvex polygonal state constraint, which is defined as the union of a set of triangles, as detailed in Section \ref{sec:homotopy-constraint}, as well as from the tether dynamics, which are hybrid due to the presence of obstacles, as discussed in Section \ref{sec:fem-model}.

\subsection{Homotopy constraint for the trajectory generation}
\label{sec:homotopy-constraint}
We describe now the formulation of the state constraints that will be used in the trajectory generation problem. The role of these constraints is twofold: first, they guarantee obstacle avoidance; second, and most important, they constrain the trajectory to lie in the same homotopy class of the candidate path $\alpha$ returned by the path planning operation performed at the previous step over the simplicial complex model $\widetilde{\mathcal{N}}$. 
To impose the homotopy constraint, we exploit the cell decomposition of $\mathcal{W}_\mathrm{free}$ that naturally emerges from the triangulation $\mathcal{T}$ computed in Section \ref{sec:modeling}, and we use it to formulate a polygonal state constraint following the approach proposed in \cite{park2015homotopy}. In particular, we constrain the points composing the trajectory to lie inside the union of the triangles crossed by the path $\alpha$, as shown in Figure \ref{fig:h-constrained-traj-planning}. 
By doing so, the location of the robot along the trajectory is constrained to remain in the same homotopy class as the candidate path $\alpha$. 

Given a path $\alpha$, we denote the set of triangles traversed by $\alpha$ as $T_\alpha = \{\tau_i\}_{i=1}^{|T_\alpha|}$, where the value of $|T_\alpha|$ depends on $\alpha$ and $\mathcal{T}$.
Each triangle $\tau_i$ imposes a convex state constraint $A_i x + b_i \leq 0$, with $A_i$ and $b_i$ representing the linear constraints corresponding to the edges of triangle $\tau_i$. 
This results in a mixed-integer state constraint of the type
\begin{equation}
    \label{eq:state_constraint_bigM}
    \begin{aligned}
        &A_i x + b_i \leq M_i (1 - \Delta_i), i\in\{1, \ldots, |T_\alpha|\},\\
        &\sum_{i=1}^{|T_\alpha|} \Delta_i = 1,
    \end{aligned}
\end{equation}
with $\Delta_i \in \{0, 1\}$ a binary variable, with the standard interpretation that $\Delta_i = 1 \Rightarrow x\in \tau_i$, and $M_i$ an appropriately large number used to render the constraint active or practically inactive, in accordance to the big-M method.

\subsection{Mixed-integer optimization-based formulation of the trajectory generation problem}
\label{sec:mixed-integer-problem}
We define now the homotopy-constrained trajectory generation problem.
The inputs of the problem are the initial tether configuration $\gamma_0$ and the candidate path $\alpha$, which starts from $\gamma_0(1)$. The goal is to compute a dynamically feasible trajectory in the same homotopy class as $\alpha$.
First, we consider a discrete-time dynamic model of the robot, denoted as
\begin{equation}
    q_\mathrm{r}(k+1) = f_\mathrm{r}(q_\mathrm{r}(k), u(k)),
\end{equation}
where $q_\mathrm{r}\triangleq [p^\top, v^\top]^\top$ is the state of the robot, composed by its position $p\in\mathbb{R}^2$ and velocity $v\in\mathbb{R}^2$, and $u$ represents the control input of the robot. A robot trajectory is then defined as
\begin{equation}
    \bm{q}_\mathrm{r} = [q^\top_\mathrm{r}(0), q^\top_\mathrm{r}(1), \ldots, q^\top_\mathrm{r}(N)]^\top,
\end{equation}
corresponding to a sequence of robot states at consecutive time instances spaced by $\Delta t$ time units, over a period of length $N \Delta t$. 
The value of $N$ is selected to be large enough based on the reference path $\alpha$, so that the robot is able to reach its goal within $N\Delta t$ time units. 
The trajectory $\bm{q}_\mathrm{r}$ is constrained to lie in $[\alpha]$ through the introduction of a set of state constraints of the form \eqref{eq:state_constraint_bigM} for each robot position $p(k), k\in\{0, ..., N\}$ along the trajectory.
The trajectory $\bm{q}_\mathrm{r}$ is fully defined by a sequence of control inputs
\begin{equation}
    \bm{u} = [u^\top(0), u^\top(1), \ldots, u^\top(N-1)]^\top.
\end{equation}
The input sequence $\bm{u}$ is the optimization variable for which we solve the trajectory generation problem, and it is penalized in the cost function through a quadratic cost term in order to promote the computation of energy-efficient trajectories.

\begin{figure}
    \centering
    \begin{subfigure}[b]{0.49\linewidth}
        \centering
        \includegraphics[width=\linewidth]{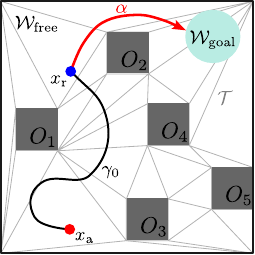}
        \caption{}
        \label{fig:h-constrained-traj-planning-a}
    \end{subfigure}
    \hfill
    \begin{subfigure}[b]{0.49\linewidth}
        \centering
        \includegraphics[width=\linewidth]{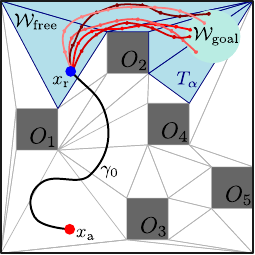}
        \caption{}
        \label{fig:h-constrained-traj-planning-b}
    \end{subfigure}
    \caption{Left: Candidate path $\alpha$ inside the base triangulation of $\mathcal{W}_\mathrm{free}$. Right: the polygonal set $T_\alpha$ used to constraint the state of the robot, defined as the union of a subset of the triangles in $\mathcal{T}_2$.}
    \label{fig:h-constrained-traj-planning}
\end{figure}

Second, we consider the dynamic model of the tether (see Section \ref{sec:fem-model}), which we denote by
\begin{equation}
    \label{eq:fem-hybrid}
    q_\mathrm{t}(k+1) = f_\mathrm{t}(q_\mathrm{t}(k), F(k)),
\end{equation}
with $F(k)$ representing the external forces acting on the nodes of the model, such as exogenous inputs (e.g., water current and gravity) and reaction forces due to the contact with the obstacle region $\mathcal{O}$.
The sequence of states of the tether is defined as
\begin{equation}
    \bm{q}_\mathrm{t} = [q^\top_\mathrm{t}(0), q^\top_\mathrm{t}(1), \ldots, q^\top_\mathrm{t}(N)]^\top,
\end{equation}
with the tether state $q_\mathrm{t}$ a spatially discretized representation of $\gamma$ at time step $k$ (see Section \ref{sec:fem-model}), defined as in \eqref{eq:tether-state}, and with $q_\mathrm{t}(0)$ the initial discretization of the tether at time 0, which we denote with $\hat{\gamma}_0$.
The tether dynamics are linked to the ones of the robot by a constraint imposing the coincidence between the location of the robot and the free endpoint of the tether. This effectively acts as a boundary condition for the finite-element model, with a second one being imposed by the anchor point, whose location is fixed.

Additionally, in order to satisfy the entanglement-free condition, we include a non-entanglement constraint with respect to the entanglement definition $\delta$ for all the configurations that the tether assumes during the motion of the robot. 
By optimizing over the control input $\bm{u}$, we can determine the trajectory of the robot, and in turn that of the tether.

The resulting trajectory generation problem is the following:
\begin{subequations}
	\label{eq:MINLP}
	\begin{align}
        J(&x_\mathrm{a}, \hat{\gamma}_0, q_{\mathrm{r}, 0}, T_\alpha) = \nonumber\\
        &\min_{\bm{u}, \bm{q}_\mathrm{r}, \bm{q}_\mathrm{t}} \lambda_\mathrm{u}\sum_{k=0}^{N-1} u^\top(k)Qu(k) \nonumber\\
        &\qquad\!\!+ \lambda_\mathrm{g} \sum_{k=1}^{N}\text{dist}\bigl(q_\mathrm{r}(k), \mathcal{W}_\mathrm{goal}\bigr)^2\label{eq:cost-function}\\
        &\text{s. t. }\; q_\mathrm{t}(0) = \hat{\gamma}_0,\label{eq:tether-1}\\
        &\qquad q_\mathrm{r}(0) = q_{\mathrm{r}, 0} \\
        &\qquad\text{for } k = 0 \ldots N-1 \nonumber\\
        &\qquad\qquad q_\mathrm{r}(k+1) = f_\mathrm{r}(q_\mathrm{r}(k), u(k))\\
        &\qquad\qquad q_\mathrm{t}(k+1) = f_\mathrm{t}(q_\mathrm{t}(k), F(k))\label{eq:tether-2}\\
        &\qquad\qquad q_\mathrm{t}^{(z)}(k+1) = q_\mathrm{r}(k+1),\label{eq:tether-3}\\
        &\qquad\qquad q_\mathrm{t}^{(0)}(k+1) = x_\mathrm{a},\label{eq:tether-4}\\
        &\qquad\qquad u(k) \in U\\
        &\qquad\qquad q_\mathrm{r}(k) \in X\\
        &\qquad q_\mathrm{r}(N) \in \mathcal{W}_\mathrm{goal} \\
        &\qquad\text{for } k = 1 \ldots N \nonumber\\
        &\qquad\qquad\delta\left(q_\mathrm{t}(k+1)\right)=0,\label{eq:tether-5}\\
        &\qquad\qquad\text{for } i=1, \ldots, |T_\alpha|\nonumber\\
        &\qquad\qquad\qquad A_i p(k) + b_i \leq M_i (1 - \Delta_i)\label{eq:h-constraint-1}\\
        &\qquad\qquad\qquad\sum_{i=1}^{|T_\alpha|} \Delta_i = 1\label{eq:h-constraint-2}
     \end{align}
\end{subequations}
Here $\lambda_\mathrm{u}, \lambda_\mathrm{g} \in \mathbb{R}_{> 0}$ are two tuning weights to balance the terms of the cost function, $Q$ is an appropriately-sized positive-definite matrix, and $\mathrm{dist}(x, \mathcal{W}_\mathrm{goal})$ is the distance between the point $x$ and the goal region. The distance can be computed either through a heuristic or as an exact geodesic in $\mathcal{W}_\mathrm{free}$.
Finally, $X$ and $U$ represents the state and input constraints, respectively, which are usually L1-norm or L2-norm constraints. 

\begin{remark}
    The computational cost of solving \eqref{eq:MINLP} highly depends on the dynamics of the tether \eqref{eq:tether-2}.
    It is worth noting that, under some classes of entanglement definitions such as those evaluating entanglement based solely on the homotopy class in which the tether lies (e.g., Definitions 3, 4, 5, 11 of \cite{battocletti2024entanglement}), the tether dynamics, and the non-entanglement constraint associated to it, can be effectively ignored.
    Under those definitions, the constraints \eqref{eq:tether-1}, \eqref{eq:tether-2}, \eqref{eq:tether-3}, \eqref{eq:tether-4}, and \eqref{eq:tether-5} can be dropped from \eqref{eq:MINLP}, significantly reducing the computational cost of solving \eqref{eq:MINLP}.
    Alternatively, the computational cost associated to the tether dynamics can be reduced by adopting a computationally efficient dynamic model of the tether, e.g., a learning-based one \cite{wang2022offline, yu2022global}, or a purely kinematic one \cite{martinez2023path}, in contrast to a lumped-mass or finite-element model. Again, we highlight that the proposed planning approach is flexible, and can accommodate the use of different dynamic models of the tether, thereby offering different solutions to alleviate the computational burden of solving \eqref{eq:MINLP}.

    We also note that, while the polygonal state constraints \eqref{eq:h-constraint-1}, \eqref{eq:h-constraint-2} are always present in the optimization problem because of the homotopy-class constraint, their impact on the computational cost of solving \eqref{eq:MINLP} can generally be limited through appropriate warm starting of the optimization variable $\bm{q}_\mathrm{r}$ based on the candidate path $\alpha$.
\end{remark}

\section{Numerical simulations}
\label{sec:case_study}

In this section, we showcase the proposed approach through a series of numerical simulations. We report the three sections of the proposed pipeline, namely, the modeling of the configuration space of the tethered robot, the path planning and disentanglement planning, and the trajectory optimization, in Sections \ref{subsec:results-s-c-model}, \ref{subsec:results-path-planning}, and \ref{subsec:results-trajectory-optimization}, respectively.

In order to demonstrate the flexibility of the proposed planning pipeline to different entanglement definitions, we consider here two entanglement definitions from \cite{battocletti2024entanglement}, namely, Definitions \ref{def:obstacle_free_conv_hull} and \ref{def:local_visibility_homotopy}.
In the following, when needed, we will indicate with a superscript $^{(\delta)}$, with $\delta\in\{\ref{def:obstacle_free_conv_hull}, \ref{def:local_visibility_homotopy}\}$, the corresponding definition being considered.

The simulations have been implemented in Python 3.11.
The simulations of Section \ref{subsec:results-s-c-model} have been run on a Linux server with 8 AMD EPYC 7252 (3.1 GHz) processors and 251 GB of RAM, while the simulations of Sections \ref{subsec:results-path-planning} and \ref{subsec:results-trajectory-optimization}, and of Appendix \ref{appendix:path-planning-timing}, have been run on a laptop with an 8th Gen Intel processor with six i7 cores (2.20 GHz) and 16 GB of RAM.
The source code is available at \href{https://github.com/gbattocletti/motion-planning-tethered-robots}{https://github.com/gbattocletti/motion-planning-tethered-robots}.

\subsection{Entanglement-free simplicial complex model}
\label{subsec:results-s-c-model}
We start by evaluating the construction of the simplicial complex model of the entanglement-free configuration space of a tethered robot, which is achieved through Algorithm \ref{alg:entanglement-free-simplicial-complex}.
The algorithm is evaluated on 6 environments with different numbers of obstacles $n\in\{1, 2, 5, 7, 12, 15\}$, with $m \leq n$ of them giving rise to multiple homotopy classes. The number of obstacles of each environment is listed in Table \ref{tab:environments-properties}. 
\begin{table}[!ht]
    \centering
    \small
    \caption{Properties of the evaluation scenarios.}
    \label{tab:environments-properties}
    \begin{tabular}{ccc}
        \toprule
        Environment & $n$ & $m$ \\ 
        \midrule
        1 &  1 &  1 \\
        2 &  2 &  2 \\
        3 &  5 &  4 \\
        4 &  7 &  6 \\
        5 & 12 &  8 \\
        6 & 15 & 10 \\
        \bottomrule
    \end{tabular}
\end{table}
All environments have size 10$\times$10 units; in each of them, the simplicial complex model is generated for three different values of $l$, namely, 10.0, 12.5, and 15.0 units of length, corresponding to 1, 1.25, and 1.5 times the environment size.
Environment 3 is depicted in Figure \ref{fig:triangulation-a}.
Two examples of the simplicial complex models constructed through Algorithm \ref{alg:entanglement-free-simplicial-complex} are shown in Figure \ref{fig:triangulation}. 
In particular, Figure \ref{fig:triangulation-c} shows the simplicial complex model $\widetilde{\mathcal{R}}$ generated from the environment in Figure \ref{fig:triangulation-a} with no entanglement constraint, while Figure \ref{fig:triangulation-d} shows the simplicial complex model $\widetilde{\mathcal{N}}^{(9)}$.

For each combination of $m$ and $l$, Algorithm \ref{alg:entanglement-free-simplicial-complex} is used to generate the length-reachable simplicial complex model $\widetilde{\mathcal{R}}$, and the entanglement-free simplicial complex models $\widetilde{\mathcal{N}}^{(6)}$ and $\widetilde{\mathcal{N}}^{(9)}$. 
As a baseline algorithm, we use the well-established approach from \cite{kim2014path}, which we use to compute the homotopy-augmented graph of the configuration space of a tethered robot, denoted as $\mathcal{H} = (\mathcal{V}, \mathcal{E})$. 
Moreover, we consider a modified version of the algorithm that, for each node $v=(x, h)$, checks if the homotopy class $h$ admits an entanglement-free tether configuration joining $x_\mathrm{a}$ and $x$. This is achieved by evaluating an entanglement definition $\delta$ on the shortest path between $x_\mathrm{a}$ and $x$ in $h$ (similarly to the procedure used to evaluate the vertices of the simplices in Section \ref{sec:simplicial-complex-model}), and results in an entanglement-free homotopy-augmented graph $\mathcal{H}^{(\delta)}$. 
The homotopy-augmented graph is computed for two different grid resolutions, namely, 0.5 and 0.25, which we indicate as a subscript of $\mathcal{H}$.

\begin{table*}[!htpb]
    \centering
    \setlength{\tabcolsep}{1.5pt}
    \setlength{\aboverulesep}{0pt}
    \setlength{\belowrulesep}{0pt}
    \caption{Comparison of simplicial complex models and homotopy-augmented graph model under different entanglement definitions.}
    \label{tab:comparison}
    \begin{tabular}{ccrrrrrrrrrrrrrrrrrr}
\toprule
\vspace{-0.25cm}
& & & & & & & & & & & & & & & & & & & \\
& & \multicolumn{2}{c}{$\widetilde{\mathcal{R}}$} & \multicolumn{2}{c}{$\widetilde{\mathcal{N}}^{(6)}$} & \multicolumn{2}{c}{$\widetilde{\mathcal{N}}^{(9)}$} & \multicolumn{2}{c}{$\mathcal{H}_{0.5}$} & \multicolumn{2}{c}{$\mathcal{H}^{(6)}_{0.5}$} & \multicolumn{2}{c}{$\mathcal{H}^{(9)}_{0.5}$} & \multicolumn{2}{c}{$\mathcal{H}_{0.25}$} & \multicolumn{2}{c}{$\mathcal{H}^{(6)}_{0.25}$} & \multicolumn{2}{c}{$\mathcal{H}^{(9)}_{0.25}$}\\
Env & $l$ & $\widetilde{\mathcal{R}}_2$ & $t$ [s] & $\widetilde{\mathcal{N}}_2$ & $t$ [s] & $\widetilde{\mathcal{N}}_2$ & $t$ [s] & $\mathcal{V}$ & $t$ [s] & $\mathcal{V}$ & $t$ [s] & $\mathcal{V}$ & $t$ [s] & $\mathcal{V}$ & $t$ [s] & $\mathcal{V}$ & $t$ [s] & $\mathcal{V}$ & $t$ [s]\\
\vspace{-0.25cm}
& & & & & & & & & & & & & & & & & & & \\
\midrule
1 & 10.0 &    4 &  0.01 &  4 &  0.01 &   4 &  0.01 &   493 &   0.35 & 320 &   0.52 & 449 &   1.98 &   1942 &    2.10 & 1203 &    3.18 & 1758 &    25.03 \\
1 & 12.5 &   10 &  0.01 &  4 &  0.01 &   8 &  0.01 &   722 &   0.52 & 320 &   0.83 & 506 &   2.95 &   2807 &    3.24 & 1203 &    5.19 & 1945 &    34.22 \\
1 & 15.0 &   14 &  0.01 &  4 &  0.01 &   8 &  0.01 &   990 &   0.74 & 320 &   1.25 & 512 &   3.55 &   3845 &    4.82 & 1203 &    8.13 & 1955 &    38.05 \\
2 & 10.0 &   18 &  0.01 &  6 &  0.01 &  14 &  0.01 &   679 &   0.57 & 317 &   0.85 & 525 &   4.65 &   2655 &    3.41 & 1187 &    5.03 & 1961 &    54.83 \\
2 & 12.5 &   18 &  0.01 &  6 &  0.02 &  14 &  0.02 &  1080 &   1.02 & 317 &   1.53 & 525 &   5.66 &   4327 &    6.56 & 1187 &    9.85 & 1961 &    61.85 \\
2 & 15.0 &   42 &  0.04 &  6 &  0.05 &  14 &  0.06 &  1724 &   1.71 & 317 &   2.70 & 525 &   8.10 &   6777 &   11.75 & 1187 &   18.17 & 1961 &    81.69 \\
3 & 10.0 &   69 &  0.08 & 14 &  0.10 &  28 &  0.13 &  1211 &   1.28 & 253 &   1.89 & 577 &  10.31 &   4660 &    7.36 &  932 &   10.83 & 2167 &   105.61 \\
3 & 12.5 &  143 &  0.19 & 14 &  0.23 &  28 &  0.31 &  2340 &   2.80 & 253 &   4.16 & 585 &  15.71 &   9014 &   18.31 &  932 &   26.45 & 2207 &   150.51 \\
3 & 15.0 &  269 &  0.40 & 14 &  0.48 &  28 &  0.71 &  4385 &   6.12 & 253 &   9.09 & 598 &  29.73 &  16884 &   49.33 &  932 &   66.17 & 2271 &   275.30 \\
4 & 10.0 &  157 &  0.24 & 14 &  0.28 &  36 &  0.39 &  2240 &   2.97 & 232 &   4.16 & 586 &  13.51 &   8935 &   20.02 &  851 &   26.69 & 2226 &   119.11 \\
4 & 12.5 &  398 &  0.70 & 14 &  0.80 &  36 &  1.21 &  5443 &   9.14 & 232 &  12.65 & 592 &  32.26 &  21743 &   78.87 &  851 &   98.08 & 2258 &   287.45 \\
4 & 15.0 &  995 &  2.05 & 14 &  2.30 &  36 &  3.60 & 13178 &  33.11 & 232 &  43.86 & 596 &  93.40 &  53389 &  747.40 &  851 &  757.32 & 2283 &  1251.06 \\
5 & 10.0 &  449 &  0.81 & 10 &  0.97 &  96 &  1.57 &  2288 &   3.38 &  76 &   4.80 & 665 &  45.31 &   8881 &   20.34 &  263 &   28.94 & 2264 &   625.08 \\
5 & 12.5 & 1314 &  2.84 & 10 &  3.30 & 119 &  5.79 &  6548 &  12.48 &  76 &  17.32 & 719 & 122.50 &  25362 &  112.49 &  263 &  135.78 & 2943 &  1640.19 \\
5 & 15.0 & 4034 & 10.80 & 10 & 12.20 & 140 & 21.70 & 18176 &  57.96 &  76 &  75.67 & 748 & 348.50 &  70953 & 1666.67 &  263 & 1640.04 & 3125 &  5370.24 \\
6 & 10.0 &  637 &  1.30 & 26 &  1.51 &  79 &  2.42 &  3471 &   6.26 & 157 &   8.43 & 542 &  43.64 &  13135 &   38.04 &  541 &   51.03 & 2035 &   412.71 \\
6 & 12.5 & 2449 &  6.10 & 26 &  6.93 &  85 & 12.00 & 11216 &  29.30 & 157 &  38.11 & 583 & 128.80 &  42704 &  390.95 &  541 &  425.92 & 2195 &  1343.98 \\
6 & 15.0 & 9004 & 29.92 & 26 & 33.40 &  88 & 57.18 & 36016 & 217.76 & 157 & 270.27 & 587 & 540.98 & 139238 & 7660.32 &  541 & 8019.35 & 2211 & 10479.63 \\%
\bottomrule
\end{tabular}
\end{table*}

An extensive testing of the algorithm is reported in Table \ref{tab:comparison}.
Table \ref{tab:comparison} reports the computation time required to build each model, and the memory occupancy, expressed as number of simplices for $\widetilde{\mathcal{R}}$ and $\widetilde{\mathcal{N}}^{(\delta)}$ (indicated as the cardinality of the set of 2-simplices), and number of nodes for $\mathcal{H}$ and $\mathcal{H}^{(\delta)}$.
The results in Table \ref{tab:comparison} highlight the benefits of the proposed approach over the baseline algorithm from \cite{kim2014path}. In particular, the proposed model can be computed in a fraction of the time required by the homotopy-augmented graph $\mathcal{H}$, which, even with a coarse resolution equal to 0.5, requires a long computation time, incompatible with real-time planning, to build. 
Furthermore, the number of simplices composing the simplicial complex model is significantly smaller than the number of nodes in $\mathcal{H}$. This results in a smaller memory footprint of the simplicial complex model, despite it being a continuous model of the robots' configuration space, and in a smaller size of the graph that path planning algorithms need to search on.
A direct consequence of the larger number of nodes composing $\mathcal{H}$ can also be found in the higher average computation time required to solve path planning queries on $\mathcal{H}$ with respect to that required to solve the same problem on $\widetilde{\mathcal{R}}$, as discussed in Appendix \ref{appendix:path-planning-timing}.

Additionally, Table \ref{tab:comparison} highlights how the introduction of non-entanglement constraints during the construction of the simplicial complex model $\widetilde{\mathcal{N}}^{(\delta)}$ and of the graph $\mathcal{H}^{(\delta)}$ result in a reduction in the size of the feasible configuration space with respect to the unconstrained case. 
Table \ref{tab:percentage-areas} reports the numerical values of the area restriction, expressed as percentage of area covered by $\widetilde{\mathcal{N}}^{(\delta)}$ (i.e., the underlying space $\mathcal{U}(\widetilde{\mathcal{N}}^{(\delta)})$) with respect to that covered by $\widetilde{\mathcal{R}}$. 
This value has been evaluated, in addition to Definitions \ref{def:obstacle_free_conv_hull} and \ref{def:local_visibility_homotopy}, also for Definitions 1, 2, 7, and 11 of \cite{battocletti2024entanglement}.\footnote{For the other definitions of \cite{battocletti2024entanglement} this evaluation is not applicable, as they either require an obstacle-free multi-robot setting, or require a closed tether configuration, or depend on tunable parameters that can make the result be any desired value between 0\% and 100\%.} The results indicate that Definitions 1, 2, and 7 have the same area coverage as Definition \ref{def:obstacle_free_conv_hull} (3rd column of Table \ref{tab:percentage-areas}), while Definition 11 matches that of Definition \ref{def:local_visibility_homotopy} (4th column of Table \ref{tab:percentage-areas}).
\begin{table}[!ht]
    \centering
    \small
    \setlength{\tabcolsep}{4.5pt}
    \caption{\vspace{0.1cm}Ratio between $\mathcal{U}(\widetilde{\mathcal{N}}^{(\delta)})$ and $\mathcal{U}(\widetilde{\mathcal{R}})$.}
    \label{tab:percentage-areas}
    \begin{tabular}{ccrr}
        \toprule
        Env & $l$ & $\frac{\mathcal{U}(\tilde{\mathcal{N}}^{(1, 6, 7)})}{\mathcal{U}(\tilde{\mathcal{R}})}$\% & $\frac{\mathcal{U}(\tilde{\mathcal{N}}^{(9, 11)})}{\mathcal{U}(\tilde{\mathcal{R}})}$\%  \\ 
        \midrule
        1 & 10 & 100.00 & 100.00 \\
        1 & 15 &  36.73 &  65.31 \\
        2 & 10 &  44.00 &  92.00 \\
        2 & 15 &  17.46 &  36.51 \\
        3 & 10 &  24.76 &  51.21 \\
        3 & 15 &   6.00 &  12.41 \\
        4 & 10 &  10.92 &  24.63 \\
        4 & 15 &   4.33 &   9.77 \\
        5 & 10 &   1.95 &  26.67 \\
        5 & 15 &   0.69 &  12.14 \\
        6 & 10 &   4.38 &  13.69 \\
        6 & 15 &   0.31 &   1.14 \\
        \bottomrule
    \end{tabular}
\end{table}

\subsection{Path planning and disentangling path planning}
\label{subsec:results-path-planning}
Once the simplicial complex model $\widetilde{\mathcal{N}}^{(\delta)}$ of the entan\-gle\-ment-free configuration space is computed, it can be used for path planning purposes. As mentioned in Section \ref{sec:path_planning}, the simplicial complex $\widetilde{\mathcal{N}}^{(\delta)}$ yields a primal graph $\widetilde{\mathcal{G}}_{\tilde{\mathcal{N}}}$ and a dual graph $\widetilde{\mathcal{G}}_{\tilde{\mathcal{N}}}'$, both suitable for path planning purposes. In Figure \ref{fig:path-planning} we show two examples of path planning using the primal graph, which results in the computation of the shortest path in each entanglement-admissible homotopy class.
In the figure, we consider two different scenarios, and we compare path planning performed on $\widetilde{\mathcal{R}}$ (Figures \ref{fig:path-planning-a} and \ref{fig:path-planning-c}) and on $\widetilde{\mathcal{N}}^{(9)}$ (Figures \ref{fig:path-planning-b} and \ref{fig:path-planning-d}). 
\begin{figure*}
    \centering
    \begin{subfigure}[b]{0.24\linewidth}
        \centering
        \includegraphics[width=\linewidth]{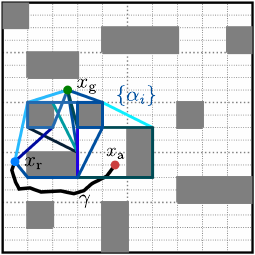}
        \caption{}
        \label{fig:path-planning-a}
    \end{subfigure}
    \hfill
    \begin{subfigure}[b]{0.24\linewidth}
        \centering
        \includegraphics[width=\linewidth]{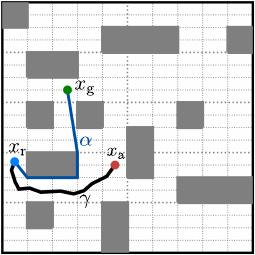}
        \caption{}
        \label{fig:path-planning-b}
    \end{subfigure}
    \hfill
    \begin{subfigure}[b]{0.24\linewidth}
        \centering
        \includegraphics[width=\linewidth]{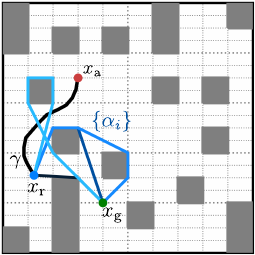}
        \caption{}
        \label{fig:path-planning-c}
    \end{subfigure}
    \hfill
    \begin{subfigure}[b]{0.24\linewidth}
        \centering
        \includegraphics[width=\linewidth]{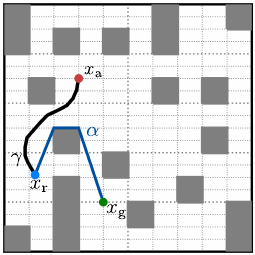}
        \caption{}
        \label{fig:path-planning-d}
    \end{subfigure}
    \caption{Example of path planning on the simplicial complex model. Two scenarios are considered. The first is shown in (a) and (b), the second in (c) and (d). The plots in (a) and (c) depict the paths found by running the path planning algorithm on $\widetilde{\mathcal{R}}$. The plots in (b) and (d) show the unique entanglement-free path found on $\widetilde{\mathcal{N}}^{(9)}$.}
    \label{fig:path-planning}
\end{figure*}
In the case of $\widetilde{\mathcal{R}}$, the path planning algorithm returns a set $\{\alpha_i\}$ containing several paths. However, the unique path found on $\widetilde{\mathcal{N}}^{(9)}$ indicates that all but one of those paths lead to an entangled tether configuration.
An example of this is shown in Figure \ref{fig:path-planning-resulting-tether}, where the tether configuration resulting from selecting the shortest path from those depicted in Figure \ref{fig:path-planning-a} is compared with that resulting from the entanglement-free path shown in Figure \ref{fig:path-planning-b}.
In particular, we note that the tether configuration in Figure \ref{fig:path-planning-resulting-tether-a} is entangled with respect to Definition \ref{def:local_visibility_homotopy}, while the one in Figure \ref{fig:path-planning-resulting-tether-b} is not.
\begin{figure}
    \centering
    \begin{subfigure}[b]{0.49\linewidth}
        \centering
        \includegraphics[width=\linewidth]{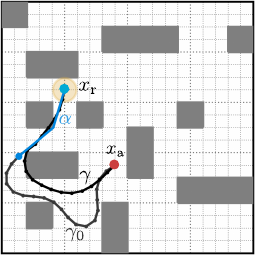}
        \caption{}
        \label{fig:path-planning-resulting-tether-a}
    \end{subfigure}
    \hfill
    \begin{subfigure}[b]{0.49\linewidth}
        \centering
        \includegraphics[width=\linewidth]{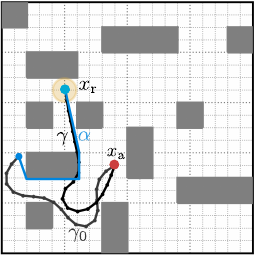}
        \caption{}
        \label{fig:path-planning-resulting-tether-b}
    \end{subfigure}
    \caption{Homotopy class of the tether after the motion of the robot along the shortest path computed on $\widetilde{\mathcal{R}}$ and on $\widetilde{\mathcal{N}}^{(9)}$ from the initial condition depicted in Figures \ref{fig:path-planning-a} and \ref{fig:path-planning-b}. The tether in (a) is entangled with respect to Definition \ref{def:local_visibility_homotopy}, as its configuration does not belong to $\widetilde{\mathcal{N}}^{(9)}$, while the one in (b) is not entangled.}
    \label{fig:path-planning-resulting-tether}
\end{figure}

The existence of multiple solutions to the path planning problem is not unique to $\widetilde{\mathcal{R}}$. In fact, depending on the entanglement definition $\delta$ and on the environment shape and goal location, multiple entanglement-free paths can exist, as shown in Figure \ref{fig:path-planning-multiplicity}, where multiple entanglement-free solutions are computed. The proposed path planning approach is capable of efficiently enumerating all the possible entanglement-free paths, with path planning queries on $\widetilde{\mathcal{N}}^{(9)}$ requiring on average 0.001s to be solved. The different paths can then be evaluated and compared to select the one to follow. 
\begin{figure}
    \centering
    \begin{subfigure}[b]{0.49\linewidth}
        \centering
        \includegraphics[width=\linewidth]{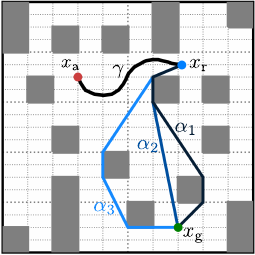}
        \caption{}
        \label{fig:path-planning-multiplicity-a}
    \end{subfigure}
    \hfill
    \begin{subfigure}[b]{0.49\linewidth}
        \centering
        \includegraphics[width=\linewidth]{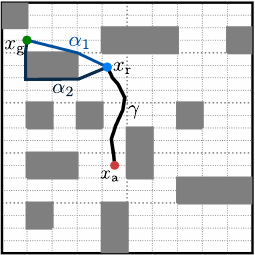}
        \caption{}
        \label{fig:path-planning-multiplicity-b}
    \end{subfigure}
    \caption{Examples of multiple entanglement-free paths computed on $\widetilde{\mathcal{N}}^{(9)}$.}
    \label{fig:path-planning-multiplicity}
\end{figure}

Lastly, we consider the disentanglement problem introduced in Problem \ref{prob:disentanglement-pp}. To this end, we perform several experiments where we select an initially entangled tether configuration, and we perform path planning on the length-reachable simplicial complex model $\widetilde{\mathcal{R}}$ until a node belonging to $\widetilde{\mathcal{N}}^{(\delta)}$ is reached, corresponding to a robot location for which the tether is not entangled. Examples of disentanglement planning are shown in Figure \ref{fig:disentanglement-planning}.
The four disentanglement scenarios depicted in Figure \ref{fig:disentanglement-planning} show how the proposed approach can be effectively used to compute disentangling paths under different entanglement definitions, a feature not available in previous path planning algorithms for tethered robots.
This capability of the proposed algorithm can enhance the safety of tethered robot systems, allowing to recover the motion capabilities of a tethered robot after it has reached an entangled tether configuration.
\begin{figure*}
    \centering
    \begin{subfigure}[b]{0.24\linewidth}
        \centering
        \includegraphics[width=\linewidth]{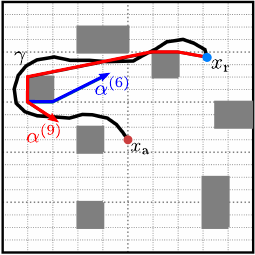}
        \caption{}
        \label{fig:disentanglement-planning-a}
    \end{subfigure}
    \hfill
    \begin{subfigure}[b]{0.24\linewidth}
        \centering
        \includegraphics[width=\linewidth]{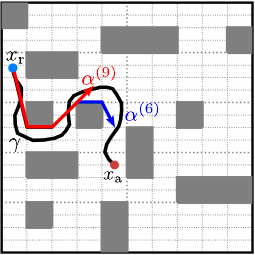}
        \caption{}
        \label{fig:disentanglement-planning-b}
    \end{subfigure}
    \begin{subfigure}[b]{0.24\linewidth}
        \centering
        \includegraphics[width=\linewidth]{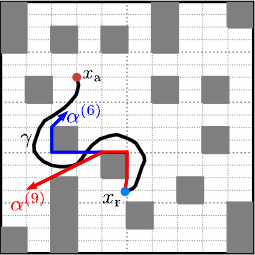}
        \caption{}
        \label{fig:disentanglement-planning-c}
    \end{subfigure}
    \hfill
    \begin{subfigure}[b]{0.24\linewidth}
        \centering
        \includegraphics[width=\linewidth]{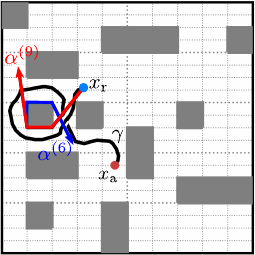}
        \caption{}
        \label{fig:disentanglement-planning-d}
    \end{subfigure}
    \caption{Example of disentanglement planning on the simplicial complex model. Four different disentanglement scenarios are depicted; in each of them, the tether is initially entangled with respect to Definitions \ref{def:obstacle_free_conv_hull} and \ref{def:local_visibility_homotopy}. For each scenario, a path going from $\widetilde{\mathcal{R}}$ to $\widetilde{\mathcal{N}}^{(\delta)}$ is shown. The \textcolor{blue}{blue} line represents the disentangling path for $\delta=6$, the \textcolor{red}{red} line the one for $\delta=$9.}
    \label{fig:disentanglement-planning}
\end{figure*}

\subsection{Entanglement-free trajectory generation}
\label{subsec:results-trajectory-optimization}

In the trajectory optimization phase, we model the robot as a double integrator, corresponding to an omnidirectional robot, where the control input is a vector $u\in\mathbb{R}^2$ representing the actuation forces in the $x$ and $y$ directions.
For space reasons, in the trajectory planning experiments we consider only one entanglement definition, namely, Definition \ref{def:local_visibility_homotopy}.
For the tether dynamics, we consider a lumped-mass model based on \cite{buckham1999dynamics}, with a collision resolution method adapted from \cite{roy2009continuous}. We remark again that the proposed planning framework is flexible, and can be used with different dynamic models of the robot and of the tether, e.g., those proposed in \cite{buckham2004development, driscoll2000development,martinez2023path}.
The values of the parameters of \eqref{eq:MINLP} used in the simulations are listed in Table \ref{tab:minlp-parameters}.
\begin{table}[!ht]
    \centering
    \small
    \caption{Parameters of \eqref{eq:MINLP} and their values used for the trajectory planning experiments.}
    \label{tab:minlp-parameters}
    \begin{tabular}{cc}
        \toprule
        Parameter & value \\ 
        \midrule
        $N$ & 30 \\
        $\Delta t$ & 0.5s \\
        $\lambda_\mathrm{g}$ & 10.0 \\
        $\lambda_\mathrm{u}$ & 2.0 \\
        $Q$ & $\mathrm{diag}(1, 1)$ \\
        $X$ & $[0, 10]\times[-2.0, 2.0]\times[0, 10]\times[-2.0, 2.0]$\\
        $U$ & $[-1, 1] \times [-1, 1]$ \\
        \bottomrule
    \end{tabular}
\end{table}

We evaluate the proposed homotopy-constrained trajectory generation approach in different planning scenarios. To this end, we run the full motion planning pipeline, up to the execution of the planned trajectories. Figure \ref{fig:fem} shows the resulting trajectory and the final tether configuration in four different scenarios. In all of them, the proposed planning pipeline is able to compute a dynamically feasible entanglement-free trajectory to reach the goal region, demonstrating the effectiveness of the proposed approach to solve Problem \ref{prob:entanglement-free-planning}. 
Furthermore, Figure \ref{fig:fem-d} shows a disentanglement planning scenario, confirming how the proposed pipeline is suitable also to solve Problem \ref{prob:disentanglement-pp}. 

\begin{figure*}
    \centering
    \begin{subfigure}[b]{0.24\linewidth}
        \centering
        \includegraphics[width=\linewidth]{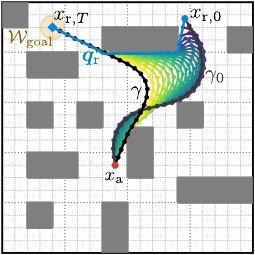}
        \caption{}
        \label{fig:fem-a}
    \end{subfigure}
    \hfill
    \begin{subfigure}[b]{0.24\linewidth}
        \centering
        \includegraphics[width=\linewidth]{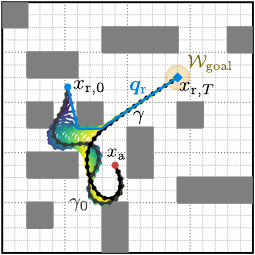}
        \caption{}
        \label{fig:fem-b}
    \end{subfigure}
    \hfill
    \begin{subfigure}[b]{0.24\linewidth}
        \centering
        \includegraphics[width=\linewidth]{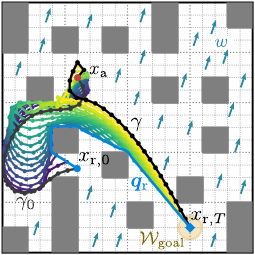}
        \caption{}
        \label{fig:fem-c}
    \end{subfigure}
    \hfill
    \begin{subfigure}[b]{0.24\linewidth}
        \centering
        \includegraphics[width=\linewidth]{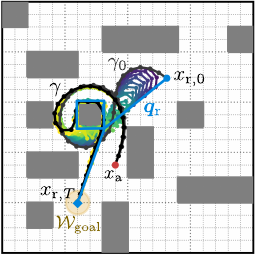}
        \caption{}
        \label{fig:fem-d}
    \end{subfigure}
    \caption{Examples of trajectory planning and execution in different scenarios. In each scenario, the robot is initially located at $x_{\mathrm{r}, 0}$ with tether configuration $\gamma_0$. A dynamically feasible trajectory $\bm{q}_\mathrm{r}$ to reach the goal region $\mathcal{W}_\mathrm{goal}$ is computed by solving \eqref{eq:MINLP}. The final robot location $x_{\mathrm{r}, T}$ and the final tether configuration $\gamma$ after a time interval $T=$ 20s are shown. The colored lines represent the intermediate tether configurations achieved during the motion of the robot, with a tether configuration being displayed for each 1s of simulation time.
    The color shade goes from blue to yellow as time progresses.
    In (c) the medium where the robot and the tether move is water, with a current $w$ acting on the tether in the direction of the blue arrows. Figure (d) shows a disentanglement planning scenario.}
    \label{fig:fem}
\end{figure*}
\begin{figure*}
    \centering
    \includegraphics[width=\textwidth]{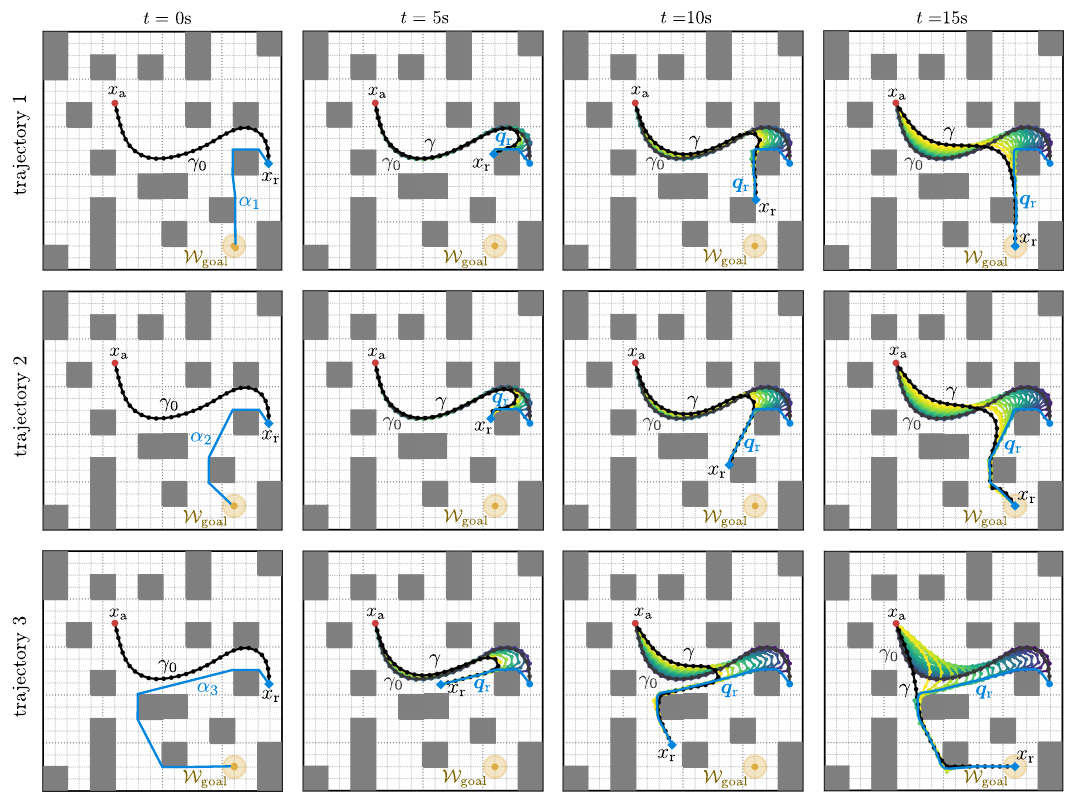}
    \caption{Comparison of three trajectories in different homotopy classes computed between the same initial and goal locations for $t=$ 0s, 5s, 10s, and 15s. The three paths $\alpha_1$, $\alpha_2$, $\alpha_3$ are similar to those computed in Figure \ref{fig:path-planning-multiplicity-a}, with slightly different initial conditions.}
    \label{fig:fem-comparison}
\end{figure*}

In the simulations, we note that solving the mixed-integer nonlinear problem (MINLP) \eqref{eq:MINLP} directly has a very high computational cost, with the solution time likely to be impractical for online planning.
To address this issue, we propose to break the trajectory optimization into two steps.
First, we relax the problem \eqref{eq:MINLP} by removing the tether dynamics by dropping the constraints \eqref{eq:tether-1}, \eqref{eq:tether-2}, \eqref{eq:tether-3}, \eqref{eq:tether-4}, and \eqref{eq:tether-5}.
The resulting optimization problem is a mixed-integer quadratic program (MIQP), which we solve to compute a dynamically feasible robot trajectory in the same homotopy class as the candidate path $\alpha$.
Then, we compute the trajectory of the tether by plugging into its dynamics \eqref{eq:fem-hybrid} the robot trajectory obtained from the MIQP relaxation, in order to verify that the trajectory is entanglement-free.
If the entanglement-free constraint is satisfied everywhere along the trajectory, the solution obtained from the MIQP relaxation of \eqref{eq:MINLP} is executed. Otherwise, the full MINLP problem \eqref{eq:MINLP} is solved, with the MIQP solution serving as an initial guess.
While this approach cannot be guaranteed to be effective for all possible tether dynamics $f_\mathrm{t}$ and for all possible entanglement definitions, in the simulations we consistently achieve entanglement-free trajectories already in the first step, as the entanglement constraint is never violated when removing the dynamics \eqref{eq:tether-2} and the other constraints associated to it.
We solve the relaxed MIQP problem with Gurobi \cite{gurobi}, achieving a solution time for the trajectory generation problem between 0.35s and 1.71s, with an average of 0.72s $\pm$ 0.46s, which is compatible with the time constraint usually associated to global motion planning problems.

In addition to the scenarios depicted in Figure \ref{fig:fem}, we showcase the parallel trajectory optimization capabilities of the proposed planning pipeline. As mentioned, path planning on $\widetilde{\mathcal{N}}^{(9)}$ returns, in general, multiple entanglement-free paths. Each of these paths can be used to generate a different homotopy-class constraint and to compute a different entanglement-free trajectory between the robot location and the goal region $\mathcal{W}_\mathrm{free}$. 
Under the proposed planning framework, the trajectories can be computed in parallel, after which they can be evaluated and compared, for example through their cost $J$, to select which one to execute.
An example of this approach is shown in Figure \ref{fig:fem-comparison}, where three different trajectories between the same initial and final robot location are shown. Each trajectory lies in a different homotopy class, and maintains the tether in an entanglement-free state with respect to Definition \ref{def:local_visibility_homotopy}.
In the example of Figure \ref{fig:fem-comparison}, the value of the cost \eqref{eq:cost-function} for the three trajectories is 5.50, 7.89, and 11.08, respectively, which leads to the selection and execution of trajectory 1.
This example demonstrates the complete functioning of the proposed planning pipeline, and its effectiveness in computing and evaluating multiple homotopically-distinct trajectories under the length and entanglement constraints posed by the tether.

\section{Conclusions}
\label{sec:conclusions}
We have introduced a novel pipeline to achieve en\-tan\-gle\-ment-free motion planning for tethered mobile robots with a slack tether.
The proposed approach consists of a modeling phase, where a computationally efficient simplicial complex model of the entanglement-free workspace is generated, followed by a path planning step and by a trajectory planning phase that return entanglement-free dynamically feasible trajectories for tethered robots with a slack tether model.
The proposed pipeline is suitable also for disentanglement planning, and is flexible with respect to the application domain, tether model, and robot dynamics.
We have demonstrated the effectiveness of the proposed pipeline in solving motion planning problems for tethered mobile robots, and the superior efficiency, in terms of computational cost and memory footprint, of the proposed approach with respect to existing methods.

Future work will focus on further improving the computational efficiency of the trajectory optimization phase, as well as on implementing and evaluating the proposed approach in 3D environments. 
In addition, the proposed planning scheme will be evaluated in real-world experiments to assess its robustness against exogenous inputs and mismatches in the dynamic models.

\printcredits  

\section*{Declaration of competing interest}
The authors declare that they have no known competing financial
interests or personal relationships that could have appeared to influence
the work reported in this paper.

\section*{Acknowledgements}
This publication has been supported by funding from the European Union's Horizon Europe Programme under grant agreement No 101093822 (SeaClear 2.0).

\section*{Data Availability}
The code used to generate the results presented in this work is available at \href{https://github.com/gbattocletti/motion-planning-tethered-robots}{https://github.com/gbattocletti/motion-planning-tethered-robots}.

\bibliographystyle{cas-model2-names}  
\bibliography{references}

\appendix
\section{Conservativeness reduction of the simplicial complex model}
\label{appendix:reduction-of-conservativeness}
As already highlighted in \cite{battocletti2025efficient}, the strategy used in Algorithm \ref{alg:entanglement-free-simplicial-complex} to generate the simplicial complex model $\widetilde{\mathcal{R}}$ can result in a conservative approximation of the real entanglement-free configuration space. 
This is a drawback of the fact that only the vertices of the triangles of the base triangulation $\mathcal{T}$ are checked. 
This means that, even if part of a triangle satisfies both the length constraint and the non-entanglement constraint, it will not be added to $\widetilde{\mathcal{R}}$ if one of its vertices do not satisfy those constraints.
We propose here an approach to reduce this conservativeness, which is based on the addition of new triangles, different from those in $\mathcal{T}$, at the boundaries of the simplicial complex model returned by Algorithm \ref{alg:entanglement-free-simplicial-complex}.
The triangles are designed to cover an additional part of the entanglement-free free workspace. In order to do so, if a vertex of a triangle being added to the simplicial complex violates one of the constraints, we search for a new vertex that satisfies them in the interior of the triangle.
The resulting triangle is added to the simplicial complex, expanding the area it covers and better approximating $\widebar{\mathcal{R}}$ and $\widebar{\mathcal{N}}$.
This operation can be repeated multiple times to reduce the approximation error of the simplicial complex model, and can be applied both to $\widetilde{\mathcal{R}}$, in which case we consider the tether length constraint, and to $\widetilde{\mathcal{N}}$, where we consider also the entanglement constraint.
Without loss of generality, in Example \ref{example:conservativeness} we illustrate the proposed approach for the case of $\widetilde{\mathcal{R}}$.

\begin{example}
    \label{example:conservativeness}
    Consider a scenario such as the one depicted in Figure \ref{fig:conservativeness-reduction}, where the triangle $\widehat{ABC}$ has two vertices that satisfy the tether length condition ($A$ and $B$, marked in green), and one that does not satisfy it ($C$, marked in red). We aim to add a triangle having the line segment $\widebar{AB}$ as base that is as large as possible. In order to do so, we find the middle point $M$ on $\widebar{AB}$, and we look for the third vertex of the triangle on the line segment $\widebar{MC}$. The location of the point $C'$ is found via binary search. The simplices corresponding to $C'$, to the edges $\widebar{AC'}$ and $\widebar{BC'}$, and to the triangle $\widehat{ABC'}$, are then added to the simplicial complex $\widetilde{\mathcal{R}}$, as shown in Figure \ref{fig:conservativeness-reduction-a}.
    The process can be further iterated to add more triangles, as shown in Figure \ref{fig:conservativeness-reduction-b}, progressively reducing the conservativeness of the simplicial complex model.
\end{example}

\begin{figure}
    \centering
    \begin{subfigure}[b]{0.49\linewidth}
        \centering
        \includegraphics[width=\linewidth]{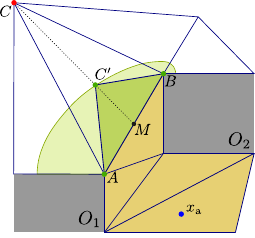}
        \caption{}
        \label{fig:conservativeness-reduction-a}
    \end{subfigure}
    \hfill
    \begin{subfigure}[b]{0.49\linewidth}
        \centering
        \includegraphics[width=\linewidth]{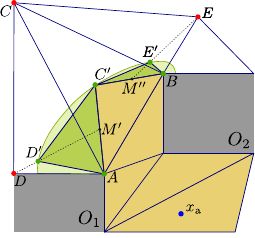}
        \caption{}
        \label{fig:conservativeness-reduction-b}
    \end{subfigure}
    \caption{Example of reduction of conservativeness of the simplicial complex model through the addition to $\widetilde{\mathcal{R}}$ and $\widetilde{\mathcal{N}}$ of triangles not originally contained in the triangulation of $\mathcal{T}$. The green shaded area represents the true reachable area under the tether length constraint, which is approximated through the triangles $\widehat{ABC'}$, $\widehat{AC'D'}$, and $\widehat{AC'E'}$.}
    \label{fig:conservativeness-reduction}
\end{figure}

\begin{remark}
    More advanced triangulation methods can also be used, such as the one proposed in \cite{gosselin2011constructing}, but at the cost of a larger increase in the number of triangles added to the simplicial complex.
    Clearly, there is a balance to be found between the computational effort of this operation, which is directly linked to the number of triangles added to the simplicial complex, and the achieved conservativeness reduction. 
    In the numerical simulations, we have found that a single iteration of the conservativeness reduction process (corresponding to the addition of one triangle at each boundary of $\widetilde{\mathcal{R}}$, as shown in Figure \ref{fig:conservativeness-reduction-a}), already leads to a significant reduction in conservativeness and has a moderate computational cost.
\end{remark}

The proposed approach to reduce the conservativeness of the simplicial complex model has been applied to the evaluation scenarios introduced in Section \ref{sec:case_study}.
Table \ref{tab:conservativeness-reduction} reports the area of the length-reachable configuration space covered by the simplicial complex model $\widetilde{\mathcal{R}}$ and that covered by the simplicial complex model $\widetilde{\mathcal{R}}'$ obtained after the conservativeness reduction step. As a baseline, we report the area covered by the homotopy-augmented grid graph $\mathcal{H}$, both for grid resolution 0.5 and 0.25. 
The table shows the benefits of the proposed conservativeness reduction procedure, significantly reducing the gap between the area of $\mathcal{U}(\widetilde{\mathcal{R}})$ and that covered by $\mathcal{H}$.
The data is also visualized in Figure \ref{fig:conservativeness-reduction-comparison}, where the data for $l=12.5$ is shown in addition to that for $l=10.0$ and $l=15.0$. 
The plot highlights how the conservativeness reduction routine yields an almost constant area increase in $\widetilde{\mathcal{R}'}$ with respect to $\widetilde{\mathcal{R}}$.
We note that the simplicial complex model, in environments with many obstacles (Env 5 and 6), covers an area larger than that of covered by $\mathcal{H}$. This is mainly due to the use of the Manhattan distance as a heuristic during the construction of $\mathcal{H}$, in order to avoid having to run a (relatively) computationally expensive curve-shortening routine for every node to compute the Euclidean distance (for more details on this strategy see \cite{kim2014path}).\footnote{While this conservativeness could be removed by computing the Euclidean distance in place of the Manhattan one, this would significantly increase the already high computational cost required to build $\mathcal{H}$.}
This results in the homotopy-augmented graph $\mathcal{H}$ being more conservative than $\mathcal{U}(\widetilde{\mathcal{R}}')$, and in some instances even than $\mathcal{U}(\widetilde{\mathcal{R}})$. 
\begin{figure}
    \centering
    \includegraphics[width=\linewidth]{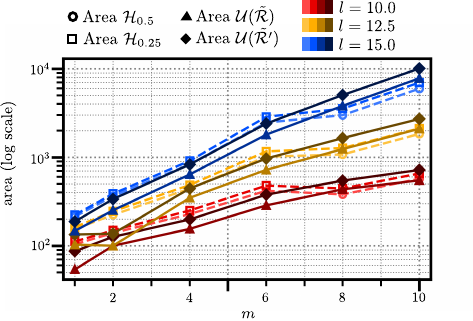}
    \caption{Comparison of the data from Table \ref{tab:conservativeness-reduction}, with the addition of that for $l=12.5$.}
    \label{fig:conservativeness-reduction-comparison}
\end{figure}
\begin{table}[!ht]
    \centering
    \small
    \setlength{\tabcolsep}{3.5pt}  
    \caption{\vspace{0.1cm}Area of $\mathcal{U}(\widetilde{\mathcal{R}})$ and $\mathcal{U}(\widetilde{\mathcal{R}}')$ compared with that covered by $\mathcal{H}$.}
    \label{tab:conservativeness-reduction}
    \begin{tabular}{ccrrrr}
        \toprule
        Env & $l$ & Area $\mathcal{U}(\tilde{\mathcal{R}})$ & Area $\mathcal{U}(\tilde{\mathcal{R}}')$ & Area $\mathcal{H}_{0.5}$ & Area $\mathcal{H}_{0.25}$ \\
        \midrule
        1 & 10.0 &   54.00 &    86.84 &  104.00 &  111.81 \\
        1 & 15.0 &  147.00 &   187.91 &  213.25 &  223.25 \\
        2 & 10.0 &  100.00 &   126.29 &  139.50 &  149.38 \\
        2 & 15.0 &  252.00 &   337.09 &  364.25 &  387.62 \\
        3 & 10.0 &  155.00 &   198.46 &  227.75 &  250.44 \\
        3 & 15.0 &  639.62 &   829.02 &  831.75 &  911.12 \\
        4 & 10.0 &  286.25 &   377.79 &  418.50 &  478.81 \\
        4 & 15.0 & 1800.75 &  2400.00 & 2469.25 & 2865.88 \\
        5 & 10.0 &  435.00 &   544.69 &  380.50 &  444.19 \\
        5 & 15.0 & 3810.00 &  5049.00 & 2988.25 & 3542.81 \\
        6 & 10.0 &  548.00 &   719.37 &  571.25 &  655.81 \\
        6 & 15.0 & 7784.50 & 10062.00 & 5959.75 & 6960.94 \\
        \bottomrule
    \end{tabular}
\end{table}

\section{Timing of path planning queries}
\label{appendix:path-planning-timing}
Owing to the efficient nature of the simplicial complex model on which path planning is performed, path planning queries on $\widetilde{\mathcal{R}}$ can be solved significantly faster than on $\mathcal{H}$, especially for higher resolutions of the homotopy-augmented graph and for higher values of $l$. 
This is demonstrated empirically in Table \ref{tab:path-planning-timing}, where we report the numerical results corresponding to the solution time of multiple graph search queries performed on $\widetilde{\mathcal{R}}$, $\mathcal{H}_{0.5}$, and $\mathcal{H}_{0.25}$. Each row of the table corresponds to the time required to complete the path planning task in a different scenario, averaged over a set of 20 path planning queries between. In each query, the planning algorithm must search a path between $x_\mathrm{a}$ and a randomly selected point in $\mathcal{W}_\mathrm{free}$.

When running path planning queries on $\widetilde{\mathcal{N}}$ and on the entanglement-constrained version of $\mathcal{H}$, the difference in the path planning computation time is less pronounced, due to the lower number of nodes to search through. 
On average, searching over the entanglement-constrained homotopy-augmented graph with resolution 0.25 takes double the amount of time compared to on $\widetilde{\mathcal{N}}$, while for resolution 0.5 it takes approximately the same amount of time as on $\widetilde{\mathcal{N}}$.
\begin{table}[t]
    \centering
    \footnotesize
    \setlength{\tabcolsep}{3.5pt}  
    \caption{Timing of path planning queries on $\widetilde{\mathcal{R}}$, $\mathcal{H}_{0.5}$, and $\mathcal{H}_{0.25}$. Each row corresponds to the average computation time over 20 path planning queries.}
    \label{tab:path-planning-timing}
    \begin{tabular}{ccccc}
        \toprule
        Env & $l$ & $t(\widetilde{\mathcal{R}})$ [s] & $t(\mathcal{H}_{0.5})$ [s] & $t(\mathcal{H}_{0.25})$ [s] \\ 
        \midrule
        5 & 10.0 & 0.0008 $\pm$ 0.0002 & 0.0007 $\pm$ 0.0002 & 0.0032 $\pm$ 0.0011 \\
        5 & 12.5 & 0.0015 $\pm$ 0.0004 & 0.0030 $\pm$ 0.0010 & 0.0105 $\pm$ 0.0021 \\
        5 & 15.0 & 0.0022 $\pm$ 0.0003 & 0.0076 $\pm$ 0.0021 & 0.0378 $\pm$ 0.0077 \\
        6 & 10.0 & 0.0010 $\pm$ 0.0002 & 0.0014 $\pm$ 0.0005 & 0.0062 $\pm$ 0.0027\\
        6 & 12.5 & 0.0017 $\pm$ 0.0002 & 0.0047 $\pm$ 0.0011 & 0.0210 $\pm$ 0.0030\\
        6 & 15.0 & 0.0041 $\pm$ 0.0004 & 0.0220 $\pm$ 0.0040 & 0.1004 $\pm$ 0.0129\\
        \bottomrule
    \end{tabular}
\end{table}
\vspace{1cm}
\vfill

\end{document}